\documentclass[sigconf]{acmart} %
\AtBeginDocument{%
  \providecommand\BibTeX{{%
    \normalfont B\kern-0.5em{\scshape i\kern-0.25em b}\kern-0.8em\TeX}}}

\copyrightyear{2021}
\acmYear{2021}
\setcopyright{acmcopyright}\acmConference[WWW '21]{Proceedings of The Web Conference 2021 (WWW ’21)}{April 19--23, 2021}{Ljubljana, Slovenia}
\acmBooktitle{Proceedings of The Web Conference 2021 (WWW ’21), August 19--23, 2021, Ljubljana, Slovenia}

\usepackage{amsmath}
\usepackage{multirow}
\usepackage{subfigure}
\usepackage{url}
\usepackage{verbatim}
\usepackage{bm}

\begin{document}
	\allowdisplaybreaks[4]
	
\title{Interpreting and Unifying Graph Neural Networks with An Optimization Framework}

\author{Meiqi Zhu}
\affiliation{Beijing University of Posts and Telecommunications}
\email{zhumeiqi@bupt.edu.cn}

\author{Xiao Wang}
\authornote{Corresponding authors}
\affiliation{Beijing University of Posts and Telecommunications}
\email{xiaowang@bupt.edu.cn}

\author{Chuan Shi}
\authornotemark[1]
\affiliation{Beijing University of Posts and Telecommunications}
\email{shichuan@bupt.edu.cn}

\author{Houye Ji}
\affiliation{Beijing University of Posts and Telecommunications}
\email{jhy1993@bupt.edu.cn}

\author{Peng Cui}
\affiliation{Tsinghua University}
\email{cuip@tsinghua.edu.cn}

\begin{abstract}
Graph Neural Networks (GNNs) have received considerable attention on graph-structured data learning for a wide variety of tasks. The well-designed propagation mechanism which has been demonstrated effective is the most fundamental part of GNNs. Although most of GNNs basically follow a message passing manner, litter effort has been made to discover and analyze their essential relations. In this paper, we establish a surprising connection between different propagation mechanisms with a unified optimization problem, showing that despite the proliferation of various GNNs, in fact, their proposed propagation mechanisms are the optimal solution optimizing a feature fitting function over a wide class of graph kernels with a graph regularization term. Our proposed unified optimization framework, summarizing the commonalities between several of the most representative GNNs, not only provides a macroscopic view on surveying the relations between different GNNs, but also further opens up new opportunities for flexibly designing new GNNs. With the proposed framework, we discover that existing works usually utilize na\"ive graph convolutional kernels for feature fitting function, and we further develop two novel objective functions considering adjustable graph kernels showing low-pass or high-pass filtering capabilities respectively. Moreover, we provide the convergence proofs and expressive power comparisons for the proposed models. Extensive experiments on benchmark datasets clearly show that the proposed GNNs not only outperform the state-of-the-art methods but also have good ability to alleviate over-smoothing, and further verify the feasibility for designing GNNs with our unified optimization framework.

\end{abstract}


\maketitle

\section{Introduction}\label{sec::Intro}
Network is a ubiquitous structure for real-world data, such as social networks, citation networks and financial networks. Recently, Graph Neural Networks (GNNs) have gained great popularity in tackling the analytics tasks \cite{kipf2017semi, wang2019attributed, errica2020a} on network-structured data. Moreover, GNNs have also been successfully applied to a wide range of application tasks, including recommendation \cite{fan2019graph, tan2020learning}, urban data mining  \cite{wang2020traffic, dai2020hybrid}, natural language processing \cite{gao2019learning, zhu2019graph} and computer vision \cite{teney2017graph, si2019an}.

The classical graph neural networks can be generally divided into two types: spectral-based GNNs and spatial-based GNNs. Spectral-based methods \cite{defferrard2016convolutional, xu2019graph} mainly focus on defining spectral graph filters via graph convolution theorem; spatial-based methods \cite{hamilton2017inductive, xu2019how} usually follow a message passing manner, where the most essential part is the feature propagation process along network topology. To date, many representative GNNs have been proposed by designing different feature propagation mechanisms, e.g., attention mechanism \cite{velickovic2018graph}, personalized pagerank \cite{klicpera2019predict} and jump connection \cite{xu2018representation}. The well-designed propagation mechanism which has been demonstrated effective is the most fundamental part of GNNs. Although there are various propagation mechanisms, they basically utilize network topology and node features through aggregating node features along network topology. In view of this, one question naturally arises: \textit{Albeit with different propagation strategies, is there a unified mathematical guideline that essentially governs the propagation mechanisms of different GNNs? If so, what is it?} A well informed answer to this question can provide a macroscopic view on surveying the relationships and differences between different GNNs in a  principled way. Such mathematical guideline, once discovered, is able to help us identify the weakness of current GNNs, and further motivates more novel GNNs to be proposed.

As the first contribution of our work, we analyze the propagation process of several representative GNNs (e.g., GCN \cite{kipf2017semi} and PPNP \cite{klicpera2019predict}), and abstract their commonalities. Surprisingly, we discover that they can be fundamentally summarized to a unified optimization framework with flexible graph convolutional kernels. The learned representation after propagation can be viewed as the optimal solution of the corresponding optimization objective implicitly. This unified framework consists of two terms: feature fitting term and graph Laplacian regularization term. The feature fitting term, building the relationship between node representation and original node features, is usually designed to meet different needs of specific GNNs. Graph Laplacian regularization term, playing the role of feature smoothing with topology, is shared by all these GNNs. For example, the propagation of GCN can be interpreted only by the graph Laplacian regularization term while PPNP needs another fitting term to constrain the similarity of the node representation and the original features.

Thanks to the macroscopic view on different GNNs provided by the proposed unified framework, the weakness of current GNNs is easy to be identified. As a consequence, the unified framework opens up new opportunities for designing novel GNNs. Traditionally, when we propose a new GNN model, we usually focus on designing specific spectral graph filter or aggregation strategy. Now, the unified framework provides another new path to achieve this, i.e., the new GNN can be derived by optimizing an objective function. In this way, we clearly know the optimization objective behind the propagation process, making the new GNN more interpretable and more reliable. Here, with the proposed framework, we discover that existing works usually utilize na\"ive graph convolutional kernels for feature fitting function, and then develop two novel flexible objective functions with adjustable kernels showing low-pass and high-pass filtering capabilities. We show that two corresponding graph neural networks with flexible graph convolutional kernels can be easily derived. Moreover, we also give the convergence ability analysis and expressive power comparisons for these two GNNs. The main contributions are summarized as follows:

\begin{itemize}
	\item We propose a unified objective optimization framework with a feature fitting function and a graph regularization term, and theoretically prove that this framework is able to unify a series of GNNs propagation mechanisms, providing a macroscopic perspective on understanding GNNs and bringing new insight for designing novel GNNs.
	\item Within the proposed optimization framework, we design two novel deep GNNs with flexible low-frequency and high-frequency filters which can well alleviate over-smoothing. The theoretical analysis on both of their convergence and excellent expressive power is provided.
	\item Our extensive experiments on series of benchmark datasets clearly show that the proposed two GNNs outperform the state-of-the-art methods. This further verifies the feasibility for designing GNNs under the unified framework.
\end{itemize}

\section{Related Work}\label{sec::Related}
\textbf{Graph Neural Networks.} \quad The current graph neural networks can be broadly divided into two categories: spectral-based GNNs and spatial-based GNNs. Spectral-based GNNs define graph convolutional operations in Fourier domain by designing spectral graph filters. \cite{bruna2014spectral} generalizes CNNs to graph signal based on the spectrum of graph Laplacian. ChebNet \cite{defferrard2016convolutional} uses Chebyshev polynomials to approximate the $K$-order localized graph filters. GCN \cite{kipf2017semi} employs the 1-order simplification of the Chebyshev filter. GWNN \cite{xu2019graph} leverages sparse and localized graph wavelet transform to design spectral GNNs. Spatial-based GNNs directly design aggregation strategies along network topology, i.e., feature propagation mechanisms. GCN \cite{kipf2017semi} directly aggregates one-hop neighbors along topology. GAT \cite{velickovic2018graph} utilizes attention mechanisms to adaptively learn aggregation weights. GraphSAGE \cite{hamilton2017inductive} uses mean/max/LSTM pooling for aggregation. MixHop \cite{abu-el-haija2019mixhop} aggregates neighbors at various distances to capture mixing relationships. GIN \cite{xu2019how} uses a simple but expressive injective multiset function for neighbor aggregation. Policy-GNN \cite{lai2020policygnn} uses a meta-policy framework to adaptively learn the aggregation policy. Furthermore, there are some advanced topics have been studied in GNNs. For example, non-Euclidean space graph neural networks \cite{ying2019hyperbolic, bachmann2020constant}; heterogeneous graph neural networks \cite{wang2019heterogeneous, zhang2019heterogeneous}; explanations for graph neural networks \cite{ying2019gnnexplainer, yuan2020xgnn}; pre-training graph neural networks \cite{hu2020strategies, qiu2020gcc}; robust graph neural networks \cite{zhu2019robust, jin2020graph}. For more details, please find in \cite{zhang2020deep, wu2020a, zhou2018graph} survey papers.

\textbf{Analysis and understanding on GNNs.} \quad Many works on understanding GNNs have been provided recently, which point out ways for designing and improving graph neural networks. Existing theoretical analysis works on GNNs are three-fold: 1) \textit{The spectral filtering characteristic analysis}: \textit{Li et al.} \cite{li2018deeper} show that the graph convolutional operation is a special form of Laplacian smoothing, and also point out the over-smoothing problem under many layers of graph convolutions; \textit{Wu et al.} \cite{wu2019simplifying} make a simplification on GCN and theoretically analyze the resulting linear model acts as a fixed low-pass filter from spectral domain; \textit{NT et al.} \cite{nt2019revisiting} also show that the graph convolutional operation is a simplified low-pass filter on original feature vectors and do not have the non-linear manifold learning property from the view of graph signal processing. 2) \textit{The over-smoothing problem analysis}: \textit{Xu et al.} \cite{xu2018representation} analyze the same over-smoothing problem by establishing the relationship between graph neural networks and random walk. And they show that GCN converges to the limit distribution of random walk as the number of layers increases; \textit{Chen et al.} \cite{chen2020simple} provide theoretical analysis and imply that nodes with high degrees are more likely to suffer from over-smoothing in a multi-layer graph convolutional model. 3) \textit{The capability of GNNs analysis}: \textit{Hou et al.} \cite{hou2020measuring} work on understanding how much performance GNNs actually gain from graph data and design two smoothness metrics to measure the quantity and quality of obtained information; \textit{Loukas et al.} \cite{loukas2020what} show that GNNs are Turing universal under sufficient conditions on their depth, width and restricted depth and width may lose a significant portion of their power; \textit{Oono et al.} \cite{oono2020graph} investigate the expressive power of GNNs as the layer size tends to infinity and show that deep GNNs can only preserve information of node degrees and connected components. However, these works do not theoretically analyze the intrinsic connections about the propagation mechanisms for GNNs.

\section{A Unified Optimization framework}\label{sec::Analysis}
\textbf{Notations.} \quad We consider graph convolutional operations on a graph $\mathcal{G} = (\mathcal{V}, \mathcal{E})$ with node set $\mathcal{V}$ and edge set $\mathcal{E}$, $n = |\mathcal{V}|$ is the number of nodes. The nodes are described by the feature matrix $\textbf{X} \in \mathbb{R}^{n\times f}$, where $f$ is the dimension of node feature. Graph structure of $\mathcal{G}$ can be described by the adjacency matrix $\textbf{A} \in \mathbb{R}^{n\times n}$ where $\textbf{A}_{i,j} = 1$ if there is an edge between nodes $i$ and $j$, otherwise 0. The diagonal degree matrix is denoted as $\textbf{D} = diag{(d_1, \cdots, d_n)}$, where $d_j = \sum_j \textbf{A}_{i,j}$. We use $\tilde{\textbf{A}} = \textbf{A} + \textbf{I}$ to represent the adjacency matrix with added self-loop and $\tilde{\textbf{D}} = \textbf{D} + \textbf{I}$. Then the normalized adjacency matrix is $\hat{\tilde{\textbf{A}}}= \tilde{\textbf{D}}^{-1/2}\tilde{\textbf{A}}\tilde{\textbf{D}}^{-1/2}$. Correspondingly, $\tilde{\textbf{L}} = \textbf{I} - \hat{\tilde{\textbf{A}}}$ is the normalized symmetric positive semi-definite graph Laplacian matrix.

\subsection{The Unified Framework}
The well-designed propagation mechanisms of different GNNs basically follow similar propagating steps, i.e, node features aggregate and transform along network topology for a certain depth. Here, we summarize the $K$-layer propagation mechanisms mainly as the following two forms.

For GNNs with layer-wise feature transformation (e.g. GCN \cite{kipf2017semi}), the $K$-layer propagation process can be represented as:
\begin{equation}\label{Ana_eq_propagation_1}
	\begin{aligned}
	\textbf{Z} = \textbf{PROPAGATE}(\textbf{X}; \mathcal{G}; K) = \bigg\langle\textbf{\textit{Trans}}\Big(\textbf{\textit{Agg}}\big\{\mathcal{G};\textbf{Z}^{(k-1)}\big\}\Big)\bigg\rangle_{K},
	\end{aligned}
\end{equation}
with $\textbf{Z}^{(0)} = \textbf{X}$ and $\textbf{Z}$ is the output representation after the $K$-layer propagation. And $\big\langle \big\rangle_{K}$, usually depending on specific GNN models, represents the generalized combination operation after $K$ convolutions. $\textbf{\textit{Agg}}\big\{\mathcal{G};\textbf{Z}^{(k-1)}\big\}$ means aggregating the $(k-1)$-layer output $\textbf{Z}^{(k-1)}$ along graph $\mathcal{G}$ for the $k$-th convolutional operation, and \textbf{\textit{Trans}($\cdot$)} is the corresponding layer-wise feature transformation operation including non-linear activation function $ReLU()$ and layer-specific learnable weight matrix $\textbf{W}$.

Some deep graph neural networks (e.g. APPNP \cite{klicpera2019predict}, DAGNN \cite{liu2020towards}) decouple the layer-wise \textbf{\textit{Trans}($\cdot$)} and $\textbf{\textit{Agg}}\big\{\mathcal{G};\textbf{Z}^{(k-1)}\big\}$, and use a separated feature transformation before the consecutive aggregation steps:
\begin{equation}\label{Ana_eq_propagation_2}
	\textbf{Z} = \textbf{PROPAGATE}(\textbf{X}; \mathcal{G}; K) = \bigg\langle\textbf{\textit{Agg}}\big\{\mathcal{G};\textbf{Z}^{(k-1)}\big\}\bigg\rangle_{K}, \\
\end{equation}
with $\textbf{Z}^{(0)} = \textbf{\textit{Trans}}(\textbf{X})$ and \textbf{\textit{Trans}($\cdot$)} can be any linear or non-linear transformation operation on original feature matrix \textbf{X}.

In addition, the combination operation $\big\langle \big\rangle_{K}$ is generally two-fold: for GNNs like GCN, SGC, and APPNP, $\big\langle \big\rangle_{K}$ directly utilizes the $K$-th layer output. And for GNNs using outputs from other layers, like JKNet and DAGNN, $\big\langle \big\rangle_{K}$ may represent pooling, concatenation or attention operations on the some (or all) outputs from $K$ layers.

Actually, the propagation process including aggregation and transformation is the key core of GNNs. Network topology and node features are the two most essential sources of information improving the learned representation during propagation: network topology usually plays the role of low-pass filter on the input node signals, which smooths the features of two connected nodes \cite{li2018deeper}. In this way, the learned node representation is able to capture the homophily of graph structure. As for the node feature, itself contains complex information, e.g., low-frequency and high-frequency information. Node feature can be flexibly used to further restrain the learned node representation. For example, APPNP adds the original node feature to the representation learned by each layer, which well preserves the personalized information so as to alleviate over-smoothing.

The above analysis implies that despite various GNNs are proposed with different propagation mechanisms, in fact, they usually potentially aim at achieving two goals: encoding useful information from feature and utilizing the smoothing ability of topology, which can be formally formulated as the following optimization objective:
\begin{equation}\label{total_framework}
	\begin{aligned}
		\mathcal{O} = \min \limits_{\textbf{Z}} {\big\{\underbrace{\zeta \big\|\textbf{F}_1 \textbf Z -  \textbf{F}_2\textbf{H} \big\|^2_F}_{\mathcal{O}_{fit}} + \underbrace{\xi tr(\textbf{Z}^T \tilde{\textbf{L}} \textbf{Z})}_{\mathcal{O}_{reg}}\big\}}.
	\end{aligned}
\end{equation}
Here, $\xi$ is a non-negative coefficient, $\zeta$ is usually chosen from [0, 1], and $\textbf{H}$ is the transformation on original input feature matrix $\textbf{X}$. $\textbf{F}_1$ and $\textbf{F}_2$ are defined as arbitrary graph convolutional kernels. $\textbf{Z}$ is the propagated representation and corresponds to the final propagation result when minimizing the objective $\mathcal{O}$.

In this unified framework, the first part $\mathcal{O}_{fit}$ is a fitting term which flexibly encodes the information in \textbf{H} to the learned representation \textbf{Z} through designing different graph convolutional kernels $\textbf{F}_1$ and $\textbf{F}_2$. Graph convolutional kernels $\textbf{F}_1$ and $\textbf{F}_2$ can be chosen from the $\textbf I$, $\hat{\tilde{\textbf{A}}}$, $\tilde{\textbf{L}}$, showing the all-pass, low-pass, high-pass filtering capabilities respectively. The second term $\mathcal{O}_{reg}$ is a graph Laplacian regularization term constraining the learned representations of two connected nodes become similar, so that the homophily property can be captured, and $\mathcal{O}_{reg}$ comes from the following graph Laplacian regularization:
\begin{equation}\label{Ana_eq_gregularization}
		\mathcal{O}_{reg} = \frac{\xi}{2} \sum_{i,j}^n \hat{\tilde{\textbf{A}}}_{i,j} \big\| \textbf{Z}_i -\textbf{Z}_j \big\|^2  = \xi tr(\textbf{Z}^T\tilde{\textbf{L}}\textbf{Z}).
\end{equation}

In the following, we theoretically prove that the propagation mechanisms of some typical GNNs are actually the special cases of our proposed unified framework as shown in Table \ref{overall_results}. This unified framework builds the connection among some typical GNNs, enabling us to interpret the current GNNs in a global perspective.

\begin{table*}[!t]
	\centering
	\caption{The overall correspondences between propagation mechanisms and optimization objectives for GNNs.}
	\label{overall_results}
	\renewcommand\arraystretch{2}
	\resizebox{\textwidth}{!}{
		\begin{tabular}{c|c|l|l}
			\hline
			\hline
			\textbf{Model}&\textbf{Characteristic}&\quad \quad \quad \quad \quad \quad \quad \quad \quad \quad \quad \quad \quad\textbf{Propagation Mechanism}&\quad \quad \quad \quad \quad \quad\quad \quad \textbf{Corresponding Objective}\\
			\hline
			\hline
			\textbf{GCN/SGC} \cite{kipf2017semi}&$K$-layer graph convolutions&$\textbf Z =  \hat{\tilde{\textbf{A}}}^{K} \textbf{X} \textbf{W*}$&$\mathcal{O} = \min \limits_{\textbf{Z}} \big\{tr(\textbf Z^{T} \tilde{\textbf{L}} \textbf Z)\big\}$, $\textbf Z^{(0)} = \textbf{X} \textbf{W*}$\\
			\hline
			\textbf{GC Operation} \cite{kipf2017semi}&$1$-layer graph convolution&$\textbf Z =  \hat{\tilde{\textbf{A}}}\textbf{X}\textbf{W}$&$\mathcal{O} = \min \limits_{\textbf{Z}} \big\{\big\|\textbf Z - \textbf{H} \big\|_F^2 + tr(\textbf Z^{T} \tilde{\textbf{L}} \textbf Z)\big\}, \textbf{H} = \textbf{X} \textbf{W}$,  (\textit{first-order})\\
			\hline
			\multirow{2}*{\textbf{PPNP/APPNP} \cite{klicpera2019predict}}&\multirow{2}*{Personalized pagerank}&\multirow{2}*{$\textbf{H} = f_{\theta}(\textbf X), \quad \left\{
				\begin{aligned}
				&\textbf{PPNP:} \quad \textbf{Z} = \alpha \big(\textbf{I} - (1-\alpha) \hat{\tilde{\textbf{A}}} \big) ^{-1} \textbf H \\
				&\textbf{APPNP:} \quad \textbf{Z} = \Big\langle (1-\alpha) \hat{\tilde{\textbf{A}}} \textbf{Z}^{(k)} + \alpha \textbf {H}\Big\rangle_{K}, \textbf{Z}^{(0)} = \textbf{H}
				\end{aligned}
				\right.$}&\multirow{2}*{$\mathcal{O} = \min \limits_{\textbf{Z}} \big\{\big\| \textbf Z - \textbf{H} \big\|_F^2 + (1/\alpha - 1) tr(\textbf Z^{T} \tilde{\textbf{L}} \textbf Z)\big\}$}\\
			~&~&~&~\\
			\hline
			\textbf{JKNet} \cite{xu2018representation}&Jumping to the last layer&$\textbf Z = \sum\limits_{k=1}^{K} \alpha_k \hat{\tilde{\textbf{A}}}^{k} \textbf{X} \textbf{W*} $&$\mathcal{O} = \min \limits_{\textbf{Z}} \big\{\big\|\textbf Z - \hat{\tilde{\textbf{A}}} \textbf H \big\|_F^2 + \xi tr(\textbf Z^{T} \tilde{\textbf{L}} \textbf Z)\big\}, \textbf{H} = \textbf X \textbf{W*}$\\
			\hline
			\textbf{DAGNN} \cite{liu2020towards}&Adaptively incorporating different layers&$\textbf{H} = f_{\theta}(\textbf X), \quad  \textbf Z = \sum\limits_{k=0}^{K} s_k \hat{\tilde{\textbf{A}}}^{k} \textbf{H}$&$\mathcal{O} = \min \limits_{\textbf{Z}} \big\{\big\|\textbf Z - \textbf{H}\big\|_F^2 + \xi tr(\textbf Z^{T} \tilde{\textbf{L}} \textbf Z)\big\}$\\
			\hline
			\multirow{3}*{\textbf{GNN-LF} (ours)}&\multirow{3}*{Flexible low-pass filtering kernel}&\multirow{3}*{$\textbf{H} = f_{\theta}(\textbf X), \quad \left\{
				\begin{aligned}
				&\textbf{closed:} \quad \textbf Z =\big\{\{\mu + 1/\alpha -1\} \textbf I + \{2-\mu-1/\alpha\} \hat{\tilde{\textbf{A}}}\big\}^{-1} \{\mu \textbf I + (1-\mu) \hat{\tilde{\textbf{A}}} \} \textbf{H}  \\
				&\textbf{iter:} \quad \textbf Z =\bigg\langle\frac{1 +\alpha\mu - 2\alpha}{1+\alpha\mu -\alpha} \hat{\tilde{\textbf{A}}} \textbf Z^{(k)} + \frac{\alpha\mu}{1+\alpha\mu -\alpha} \textbf{H} + \frac{\alpha - \alpha \mu}{1+\alpha\mu -\alpha} \hat{\tilde{\textbf{A}}} \textbf{H}\bigg\rangle_{K}, \\
				&\quad \quad \quad \textbf Z^{(0)} = \frac{\mu}{1+\alpha\mu -\alpha} \textbf{H} + \frac{1 - \mu}{1+\alpha\mu -\alpha} \hat{\tilde{\textbf{A}}} \textbf{H}
				\end{aligned}
				\right.$}&\multirow{3}*{$\mathcal{O} = \min \limits_{\textbf{Z}} \big\{ \big\| \{\mu \textbf I + (1-\mu) \hat{\tilde{\textbf{A}}} \}^{1/2} (\textbf Z - \textbf{H}) \big\|_F^2 + (1/\alpha - 1) tr(\textbf Z^{T} \tilde{\textbf{L}} \textbf Z)\big\}$}\\
			~&~&~&~\\
			~&~&~&~\\
			\hline
			\multirow{3}*{\textbf{GNN-HF} (ours)} &\multirow{3}*{Flexible high-pass filtering kernel}&\multirow{3}*{$\textbf{H} = f_{\theta}(\textbf X), \quad \left\{
				\begin{aligned}
				&\textbf{closed:} \quad \textbf Z = \big\{(\beta + 1/\alpha) \textbf I + (1 - \beta - 1/\alpha) \hat{\tilde{\textbf{A}}} \big\}^{-1} \{\textbf I + \beta \tilde{\textbf{L}} \}\textbf{H}  \\
				&\textbf{iter:} \quad \textbf Z =\bigg\langle \frac{\alpha \beta - \alpha + 1}{\alpha \beta + 1} \hat{\tilde{\textbf{A}}} \textbf Z^{(k)} + \frac{\alpha}{\alpha \beta + 1} \textbf{H} + \frac{\alpha \beta}{\alpha \beta + 1} \tilde{\textbf{L}}\textbf{H} \bigg\rangle_{K}, \\
				&\quad \quad \quad \textbf Z^{(0)} = \frac{1}{\alpha \beta + 1} \textbf{H} + \frac{\beta}{\alpha \beta + 1} \tilde{\textbf{L}} \textbf{H}
				\end{aligned}
				\right.$}&\multirow{3}*{$\mathcal{O} = \min \limits_{\textbf{Z}} \big\{\big\| \{\textbf I + \beta \tilde{\textbf{L}} \}^{1/2} (\textbf Z - \textbf{H}) \big\|_F^2 + (1/\alpha -1) tr( \textbf Z^{T} \tilde{\textbf{L}} \textbf Z)\big\}$}\\
			~&~&~&~\\
			~&~&~&~\\
			\hline
			\hline
	\end{tabular}}
\end{table*}

\subsection{Interpreting GCN and SGC}\label{subsec4::SGC}
\textbf{GCN} \cite{kipf2017semi}/\textbf{SGC} \cite{wu2019simplifying}.\quad Graph Convolutional Network (GCN) has the following propagation mechanism which conducts linear transformation and nonlinearity activation repeatedly throughout $K$ layers:
\begin{equation} \label{1_eq_GCN_propagation}
	\begin{aligned}
		\textbf Z &= \textbf{PROPAGATE}(\textbf{X}; \mathcal{G}; K)_{gcn} \\
		&=\sigma(\hat{\tilde{\textbf{A}}}(\cdots(\sigma(\hat{\tilde{\textbf{A}}}\textbf X \textbf W^{(0)})\cdots)\textbf W^{(K-1)}).\\		
	\end{aligned}
\end{equation}
Simplifying Graph Convolutional Network (SGC) reduces this excess complexity through removing nonlinearities and collapsing weight matrices between consecutive layers. The linear model exhibits comparable performance since SGC has the similar propagation mechanism with GCN as:
\begin{equation} \label{1_eq_SGC_propagation}
	\begin{aligned}
		\textbf Z &= \textbf{PROPAGATE}(\textbf{X}; \mathcal{G}; K)_{sgc}\\
		&= \hat{\tilde{\textbf{A}}}\cdots\hat{\tilde{\textbf{A}}} \textbf{X} \textbf{W}^{(0)}\textbf{W}^{(1)}\cdots\textbf{W}^{(K-1)} = \hat{\tilde{\textbf{A}}}^{K} \textbf{X} \textbf{W*},
	\end{aligned}
\end{equation}
where $\textbf{W*} = \textbf{W}^{(0)}\textbf{W}^{(1)}\cdots\textbf{W}^{(K-1)}$. We have the following interpretations on the propagation mode of SGC (GCN) under the proposed unified framework.

\begin{theorem} \label{sgc_theorem}
	With $\zeta=0$ and $\xi = 1$ in Eq. (\ref{total_framework}), the propagation process of SGC/GCN optimizes the following graph regularization term:
	\begin{equation}\label{4_eq_SGC_object}
		\mathcal{O} = \min \limits_{\textbf{Z}} \big\{tr(\textbf Z^{T} \tilde{\textbf{L}} \textbf Z)\big\},
	\end{equation}
 	where $\textbf{Z}$ is initialized as $\textbf{X}\textbf{W*}$.
\end{theorem}
\begin{proof}
	Set derivative of Eq. (\ref{4_eq_SGC_object}) with respect to \textbf{Z} to zero:
	\begin{equation} \label{4_eq_SGC_proof1}
		\begin{aligned}
			\frac{\partial tr(\textbf{Z}^T \tilde{\textbf{L}} \textbf{Z})}{\partial \textbf{Z}} = 0 \quad &\Rightarrow \quad	\tilde{\textbf{L}} \textbf{Z} = 0 \quad \Rightarrow \quad	\textbf{Z} = \hat{\tilde{\textbf{A}}} \textbf{Z}.
		\end{aligned}
	\end{equation}
   	Eq. (\ref{4_eq_SGC_proof1}) can be explained as an limit distribution where $\textbf{Z}_{lim} = \hat{\tilde{\textbf{A}}} \textbf{Z}_{lim}$. Then we use the following iterative form to approximate the limit $\textbf{Z}_{lim}$ with $K \rightarrow \infty$:
    \begin{equation}\label{4_eq_SGC_proof2}
		\textbf{Z}^{(K)} = \hat{\tilde{\textbf{A}}} \textbf{Z}^{(K-1)}.
	\end{equation}
	When SGC initializes input representation as $\textbf{Z}^{(0)}=\textbf{X} \textbf{W*}$, Eq. (\ref{4_eq_SGC_proof2}) becomes:
	\begin{equation}\label{4_eq_SGC_proof3}
		\textbf Z^{(K)} = \hat{\tilde{\textbf{A}}} \textbf Z^{(K-1)} = \hat{\tilde{\textbf{A}}}^2 \textbf Z^{(K-2)} = \cdots =  \hat{\tilde{\textbf{A}}}^K \textbf Z^{(0)} = \hat{\tilde{\textbf{A}}}^{K} \textbf{X} \textbf{W*},
	\end{equation}
	which matches the propagation mechanism of SGC. Since GCN can be simplified as SGC by ignoring the non-linear transformation, this conclusion also holds for GCN.
\end{proof}

\textbf{Graph Convolutional Operation}. \quad The above analysis is for the consecutive $K$ layers graph convolutional operations, here, we also pay attention to the one layer graph convolutional operation (\textbf{GC operation}) with the following propagation mechanism:
\begin{equation}
	\begin{aligned}
		\textbf Z =  \textbf{PROPAGATE}(\textbf{X}; \mathcal{G}; 1)_{gc} = \hat{\tilde{\textbf{A}}}\textbf{X}\textbf{W}.
	\end{aligned}
\end{equation}

\begin{theorem} \label{GC_operation}
	With $\textbf{F}_1 = \textbf{F}_2 = \textbf{I}$, $\zeta=1, \xi=1$ in Eq. (\ref{total_framework}), the 1-layer GC operation optimizes the following objective under first-order approximation:
	\begin{equation}\label{4_eq_ReGCN_object}
		\mathcal{O} = \min \limits_{\textbf{Z}} \big\{\big\|\textbf Z - \textbf{H} \big\|_F^2 + tr(\textbf Z^{T} \tilde{\textbf{L}} \textbf Z)\big\},
	\end{equation}
	where $\textbf H = \textbf X \textbf W $ is the linear transformation on feature, $\textbf W$ is a trainable weight matrix.
\end{theorem}
\begin{proof}
Please refer to Appendix \ref{subsection::GC Operations}
\end{proof}

\subsection{Interpreting PPNP and APPNP}\label{subsec1::PPNP}
\textbf{PPNP} \cite{klicpera2019predict} \quad is a graph neural network which utilizes a propagation mechanism derived from personalized PageRank and separates the feature transformation from the aggregation process:
\begin{equation}\label{1_eq_PPNP_propagation}
\begin{aligned}
\textbf{Z} &= \textbf{PROPAGATE}(\textbf{X}; \mathcal{G}; \infty)_{ppnp}\\
&= \alpha \big(\textbf{I} - (1-\alpha) \hat{\tilde{\textbf{A}}} \big) ^{-1} \textbf H, \quad and \quad \textbf H = f_{\theta}(\textbf X),\\
\end{aligned}
\end{equation}
where $\alpha \in (0, 1]$ is the teleport probability, and $\textbf H$ is the non-linear transformation result of the original feature \textbf{X} using an MLP network $f_{\theta}(\cdot)$.

Furthermore, due to the high complexity of calculating the inverse matrix, a power iterative version with linear computational complexity named APPNP is used for approximation. The propagation process of APPNP can be viewed as a layer-wise graph convolution with a residual connection to the initial transformed feature matrix $\textbf{H}$:
\begin{equation}\label{1_eq_APPNP}
\begin{aligned}
\textbf{Z} &=  \textbf{PROPAGATE}(\textbf{X}; \mathcal{G}; K)_{appnp}\\
& = \bigg\langle (1-\alpha) \hat{\tilde{\textbf{A}}} \textbf{Z}^{(k)} + \alpha \textbf {H} \bigg\rangle_K, \quad and \quad \textbf{Z}^{(0)} = \textbf{H} = f_\theta (\textbf{X}). \\
\end{aligned}
\end{equation}
Actually, it has been proved in \cite{klicpera2019predict} that APPNP converges to PPNP when $K \rightarrow \infty$, so we use one objective under the framework to explain both of them.

\begin{theorem} \label{ppnp_theorem}
With $\textbf{F}_1 = \textbf{F}_2 = \textbf{I}$,  $\zeta = 1$, $\xi = 1/\alpha - 1, \alpha \in (0, 1]$ in Eq. (\ref{total_framework}), the propagation process of PPNP/APPNP optimizes the following objective:
\begin{equation}\label{1_eq_PPNP_objective}
	 \mathcal{O} = \min \limits_{\textbf{Z}} \big\{\big\| \textbf Z - \textbf H \big\|_F^2 + \xi tr(\textbf Z^{T} \tilde{\textbf{L}} \textbf Z)\big\},
\end{equation}
where $\textbf H = f_{\theta}(\textbf X)$.
\end{theorem}

\begin{proof}
We can set the derivative of Eq. (\ref{1_eq_PPNP_objective}) with respect to \textbf{Z} to zero and get the optimal \textbf{Z} as:
\begin{equation}\label{1_proof_PPNP}
	\begin{aligned}
		\frac{\partial \Big\{\|\textbf Z - \textbf H \|_F^2 + \xi tr(\textbf Z^{T} \tilde{\textbf{L}} \textbf Z)\Big\}}{\partial \textbf Z} = 0 \quad \Rightarrow \quad \textbf Z - \textbf H + \xi \tilde{\textbf{L}} \textbf Z = 0.
	\end{aligned}
\end{equation}
Note that a matrix $\textbf{M}$ has an inverse matrix if and only if the determinant of matrix $\textbf{det}(\textbf{M})$ is not zero. Since the eigenvalues of the normalized Laplacian matrix $\lambda_i \in [0,2)$, and the eigenvalues of the matrix $\textbf I + \xi \tilde{\textbf{L}}$ are $(1 + \xi \lambda_i)>0$. Then $\textbf{det}(\textbf I + \xi \tilde{\textbf{L}})>0$ and \textbf{Z} in Eq. (\ref{1_proof_PPNP}) can be rewritten as:
\begin{equation} \label{1_eq_PPNP_closed}
\textbf Z = (\textbf I + \xi \tilde{\textbf{L}}) ^{-1} \textbf H.
\end{equation}
We use $\hat{\tilde{\textbf{A}}}$ and $\alpha$ to rewrite Eq. (\ref{1_eq_PPNP_closed}):
\begin{equation}\label{1_eq_PPNP_rewrite}
	\begin{aligned}
		\textbf{Z} = \big\{\textbf I + (1/\alpha - 1) \big({\textbf{I} - \hat{\tilde{\textbf{A}}}}\big)\big\}^{-1} \textbf H = \alpha \big(\textbf{I} - (1-\alpha) \hat{\tilde{\textbf{A}}} \big) ^{-1} \textbf H,
	\end{aligned}
\end{equation}
which exactly corresponds to the propagation mechanism of PPNP or the convergence propagation result of APPNP.
\end{proof}

\subsection{Interpreting JKNet and DAGNN}\label{subsec2::JKNet}
\textbf{JKNet} \cite{xu2018representation} \quad is a deep graph neural network which exploits information from neighborhoods of differing locality. This architecture selectively combines aggregations from different layers with Concatenation/Max-pooling/Attention at the output, i.e., the representations "jump" to the last layer.

For convenience, following \cite{wu2019simplifying, chen2020simple}, we simplify the $k$-th ($k \in [1, K]$) layer graph convolutional operation in the similar way by ignoring the non-linear activation with $\sigma(x)=x$ and sharing $\textbf{W*} = \textbf{W}^{(0)}\textbf{W}^{(1)}\cdots\textbf{W}^{(k-1)}$ for each layer. Then $k$-th layer temporary output is $\hat{\tilde{\textbf{A}}}^k \textbf{X} \textbf{W*}$. Using attention mechanism for combination at the last layer, the $K$-layer propagation result of JKNet can be written as:
\begin{equation} \label{2_eq_JKNet_propagation}
	\begin{aligned}
		\textbf{Z} &=  \textbf{PROPAGATE}(\textbf{X}; \mathcal{G}; K)_{JKNet}\\
		& = \alpha_1 \hat{\tilde{\textbf{A}}} \textbf{X}\textbf{W*} + \alpha_2 \hat{\tilde{\textbf{A}}}^{2} \textbf{X}\textbf{W*}  + \cdots + \alpha_K \hat{\tilde{\textbf{A}}}^{K} \textbf{X}\textbf{W*}  = \sum\limits_{k=1}^{K} \alpha_k \hat{\tilde{\textbf{A}}}^{k} \textbf{X} \textbf{W*},
	\end{aligned}
\end{equation}
where $\alpha_1, \alpha_2, ..., \alpha_K$ are the learnable fusion weights with $\sum\limits_{k=1}^K{\alpha_k}=1$, and for convenient analysis, we assume that all nodes of the $k$-th layer share one common weight $\alpha_k$.

\begin{theorem}
	With $\textbf{F}_1 = \textbf{I}$, $\textbf{F}_2 = \hat{\tilde{\textbf{A}}}$,  $\zeta = 1$, and  $\xi \in (0, \infty)$ in Eq. (\ref{total_framework}), the propagation process of JKNet optimizes the following objective:
	\begin{equation}\label{2_eq_JKNet_object}
		\mathcal{O} = \min \limits_{\textbf{Z}} \big\{\big\|\textbf Z - \hat{\tilde{\textbf{A}}} \textbf{H} \big\|_F^2 + \xi tr(\textbf Z^{T} \tilde{\textbf{L}} \textbf Z)\big\},
	\end{equation}
here $\textbf{H} = \textbf{X}\textbf{W*}$ is the linear feature transformation after simplifications.
\end{theorem}

\begin{proof}
Similarly, we can set derivative of Eq. (\ref{2_eq_JKNet_object}) with respect to \textbf{Z} to zero and get the optimal \textbf{Z} as:
\begin{equation} \label{2_eq_JKNet_proof1}
	\begin{aligned}
		\frac{\partial \big\{\|\textbf Z - \hat{\tilde{\textbf{A}}} \textbf{H} \|_F^2 + \xi tr(\textbf Z^{T} \tilde{\textbf{L}} \textbf Z)\big\}}{\partial \textbf Z} = 0 \quad \Rightarrow \quad \textbf Z - \hat{\tilde{\textbf{A}}} \textbf{H} + \xi \tilde{\textbf{L}} \textbf Z = 0.
	\end{aligned}
\end{equation}
Note that $\textbf{det}(\textbf{I} + \xi \tilde{\textbf{L}}) >0 $, thus matrix $\{\textbf{I} + \xi \tilde{\textbf{L}}\}^{-1}$ exists. Then the corresponding closed-form solution can be written as:
\begin{equation} \label{2_eq_JKNet_proof3}
	\textbf Z  = \big\{(1+\xi)\textbf{I} - \xi \hat{\tilde{\textbf{A}}} \big\}^{-1} \hat{\tilde{\textbf{A}}} \textbf{H}.
\end{equation}
Since $\frac{\xi}{1+\xi} < 1$ for $\forall \xi > 0$, and matrix $\hat{\tilde{\textbf{A}}}$ has absolute eigenvalues bounded by 1, thus, all its positive powers have bounded operator norm, then the inverse matrix can be decomposed as follows with $k \rightarrow \infty$:
\begin{equation} \label{2_eq_JKNet_proof4}
	\begin{aligned}
	\textbf Z  &= \frac{1}{1+\xi} \big\{\textbf{I} - \frac{\xi}{1+\xi} \hat{\tilde{\textbf{A}}} \big\}^{-1} \hat{\tilde{\textbf{A}}} \textbf{H} \\
	&= \frac{1}{1+\xi}  \Big\{\textbf{I} + \frac{\xi}{1+\xi} \hat{\tilde{\textbf{A}}} + \frac{\xi^2}{(1+\xi)^2} \hat{\tilde{\textbf{A}}}^2 + \cdots + \frac{\xi^K}{(1+\xi)^K} \hat{\tilde{\textbf{A}}}^K + \cdots \Big\} \hat{\tilde{\textbf{A}}} \textbf{H}. \\
	\end{aligned}
\end{equation}
With $\textbf{H} = \textbf{X}\textbf{W*}$, we have the following expansion:
\begin{equation} \label{2_eq_JKNet_proof5}
	\textbf Z  = \frac{1}{1+\xi} \hat{\tilde{\textbf{A}}} \textbf{X}\textbf{W*} + \frac{\xi}{(1+\xi)^2} \hat{\tilde{\textbf{A}}}^2 \textbf{X}\textbf{W*} + \cdots + \frac{\xi^{K-1}}{(1+\xi)^{K}} \hat{\tilde{\textbf{A}}}^K \textbf{X}\textbf{W*} + \cdots.
\end{equation}
Note that $\frac{1}{1+\xi} + \frac{\xi}{(1+\xi)^2} + \cdots + \frac{\xi^{K-1}}{(1+\xi)^{K}} + \cdots = 1$ and we can change the coefficient $\xi \in (0, \infty)$ to fit fusion weights $\alpha_1, \alpha_2, \cdots, \alpha_K$. When the layer $K$ is large enough, the propagation mechanism of JKNet in Eq. (\ref{2_eq_JKNet_propagation}) approximately corresponds to the objective Eq. (\ref{2_eq_JKNet_object}).
\end{proof}

\textbf{DAGNN} \cite{liu2020towards}. \quad Deep Adaptive Graph Neural Networks (DAGNN) tries to adaptively incorporate information from large receptive fields. After decoupling representation transformation and propagation, the propagation mechanism of DAGNN is similar to that of JKNet:
\begin{equation} \label{3_eq_DAGNN_propagation}
	\begin{aligned}
		\textbf{Z} &=  \textbf{PROPAGATE}(\textbf{X}; \mathcal{G}; K)_{DAGNN} \\
		& = s_0 \textbf H + s_1 \hat{\tilde{\textbf{A}}} \textbf{H} + s_2 \hat{\tilde{\textbf{A}}}^{2} \textbf{H}  + \cdots + s_K \hat{\tilde{\textbf{A}}}^{K} \textbf{H} \\
		& = \sum\limits_{k=0}^{K} s_k \hat{\tilde{\textbf{A}}}^{k} \textbf{H}, \quad and \quad \textbf{H} = f_{\theta}(\textbf{X}).
	\end{aligned}
\end{equation}
$\textbf{H} = f_{\theta}(\textbf{X})$ is the non-linear feature transformation using an MLP network, which is conducted before the propagation process, and $s_0, s_1, \cdots, s_K$ are the learnable retainment scores where $\sum\limits_{k=0}^K{s_k}=1$ and we assume that all nodes of the $k$-th layer share one common weight $s_k$ for convenience.
\begin{theorem}
	With $\textbf{F}_1 = \textbf{F}_2 = \textbf{I}$, $\zeta = 1$ and $\xi \in (0, \infty)$ in Eq. (\ref{total_framework}), the propagation process of DAGNN optimizes the following objective:
	\begin{equation}\label{2_eq_DAGNN_object}
		\mathcal{O} = \min \limits_{\textbf{Z}} \big\{\big\|\textbf Z - \textbf{H} \big\|_F^2 + \xi tr(\textbf Z^{T} \tilde{\textbf{L}} \textbf Z)\big\},
	\end{equation}
	where $\textbf{H} = f_{\theta}(\textbf X)$ is the non-linear transformation on feature matrix, the retainment scores $s_0, s_1, \cdots, s_K$ are approximated by $\xi \in (0, \infty)$.
\end{theorem}
\begin{proof}
	We also set derivative of Eq. (\ref{2_eq_DAGNN_object}) with respect to \textbf{Z} to zero and get the closed-form solution as:
	\begin{equation} \label{2_eq_DAGNN_proof3}
		\textbf Z  = \big\{(1+\xi)\textbf{I} - \xi \hat{\tilde{\textbf{A}}} \big\}^{-1} \textbf{H}.
	\end{equation}
	Through the decomposition process similar to JKNet, we have the following expansion:
	\begin{equation} \label{2_eq_JKNet_proof5}
		\textbf Z  = \frac{1}{1+\xi} \textbf{H} + \frac{\xi}{(1+\xi)^2} \hat{\tilde{\textbf{A}}} \textbf{H} + \cdots + \frac{\xi^{K}}{(1+\xi)^{K+1}} \hat{\tilde{\textbf{A}}}^{K} \textbf{H} + \cdots.
	\end{equation}
	Note that we can change $\xi \in (0, \infty)$ to fit the retainment scores where $\frac{1}{1+\xi} + \frac{\xi}{(1+\xi)^2} + \cdots + \frac{\xi^{K-1}}{(1+\xi)^{K}} + \cdots = 1$. Then the propagation mechanism of DAGNN approximately corresponds to the objective Eq. (\ref{2_eq_DAGNN_object}).
\end{proof}
\subsection{Discussion}
For clarity, we conclude the overall relations between different GNNs and the corresponding objective functions in Table \ref{overall_results}. It can be seen that our proposed framework abstracts the commonalities between different promising representative GNNs. Based on the framework, we can understand their relationships much easier. For example, the corresponding optimization objective for SGC in Theorem \ref{sgc_theorem} only has a graph regularization term, while the objective for APPNP in Theorem \ref{ppnp_theorem} has both fitting term and graph regularization term. The explicit difference of objective function well explains deep APPNP (PPNP) outperforms SGC (GCN) on over-smoothing problem by additionally requiring the learned representation to encode the original features.

On the other hand, our proposed framework shows a big picture of GNNs by mathematically modelling the objective optimization function. Considering that different existing GNNs can be fit into this framework, novel variations of GNNs can also be easily come up. All we need is to design the variables within this framework (e.g., different graph convolutional kernels $\textbf{F}_1$ and $\textbf{F}_2$) based on the specific scenarios, the corresponding propagation can be easily derived, and new GNNs architecture can be naturally designed. With one targeted objective function, the newly designed model is more interpretable and more reliable.

\section{GNN-LF/HF: Our Proposed Models}\label{sec::OPGNN}
Based on the unified framework, we find that most of the current GNNs simply set $\textbf{F}_1$ and $\textbf{F}_2$ as $\textbf{I}$ in feature fitting term, implying that they require all original information in $\textbf{H}$ to be encoded into \textbf{Z}. However, in fact, the $\textbf{H}$ may inevitably contain noise or uncertain information. We notice that JKNet has the propagation objective with $\textbf{F}_2$ as $\hat{\tilde{\textbf{A}}}$, which can encode the low-frequency information in $\textbf{H}$ to $\textbf{Z}$. While, in reality, the situation is more complex because it is hard to determine what information should be encoded, only considering one type of information cannot satisfy the needs of different downstream tasks, and sometimes high-frequency or all information is even also helpful. In this section, we focus on designing novel $\textbf{F}_1$ and $\textbf{F}_2$ to flexibly encode more comprehensive information under the framework.

\subsection{GNN with Low-pass Filtering Kernel}
\subsubsection{Objective Function} \quad  Here we first consider building the relationship of $\textbf{H}$ and $\textbf{Z}$ in both original and low-pass filtering spaces.
\begin{theorem}
	With $\textbf{F}_1 = \textbf{F}_2 =  \{\mu \textbf I + (1-\mu) \hat{\tilde{\textbf{A}}} \}^{1/2}, \mu \in [1/2,1)$, $\zeta = 1$ and $\xi = 1/\alpha - 1, \alpha \in (0, 2/3)$ in Eq. (\ref{total_framework}), the propagation process considering flexible low-pass filtering kernel on feature is:
	\begin{equation}\label{1_eq_V1obj3}
		\mathcal{O} = min  \big\{\big\| \{\mu \textbf I + (1-\mu) \hat{\tilde{\textbf{A}}} \}^{1/2} (\textbf Z - \textbf H) \big\|_F^2 + \xi tr(\textbf Z^{T} \tilde{\textbf{L}} \textbf Z)\big\}, \\
	\end{equation}
	where $\textbf{H}= f_{\theta}(\textbf X)$.
\end{theorem}

Note that $\mu$ is a balance coefficient, and we set $\mu \in [1/2,1)$ so that $\mu \textbf I + (1-\mu) \hat{\tilde{\textbf{A}}} = \textbf{\textit{V}} \pmb {\Lambda} \textbf{\textit{V}}^T$ is a symmetric and positive semi-definite matrix. Therefore, the matrix $\{\mu \textbf I + (1-\mu) \hat{\tilde{\textbf{A}}} \}^{1/2} = \textbf{\textit{V}} \pmb {\Lambda}^{1/2} \textbf{\textit{V}}^T$ has a filtering behavior similar to that of $\mu \textbf I + (1-\mu) \hat{\tilde{\textbf{A}}}$ in spectral domain. And we set $\alpha \in (0, 2/3)$ to ensure the iterative approximation solution in subsection \ref{4.1.3} has positive coefficients. By adjusting the balance coefficient $\mu$, the designed objective can flexibly constrain the similarity of $\textbf{Z}$ and $\textbf{H}$ in both original and low-pass filtering spaces, which is beneficial to meet the needs of different tasks.

\subsubsection{Closed Solution} \label{4.1.2}\quad To minimize the objective function in Eq. (\ref{1_eq_V1obj3}), we set  derivative of Eq. (\ref{1_eq_V1obj3}) with respect to \textbf{Z} to zero and derive the corresponding closed-form solution as follows:
\begin{equation}\label{1_eq_V1close1}
\textbf Z = \{\mu \textbf I + (1-\mu) \hat{\tilde{\textbf{A}}} + (1/\alpha -1) \tilde{\textbf{L}} \} ^{-1} \{\mu \textbf I + (1-\mu) \hat{\tilde{\textbf{A}}} \} \textbf H.
\end{equation}
We can rewrite the Eq. (\ref{1_eq_V1close1}) using $\hat{\tilde{\textbf{A}}}$ as:
\begin{equation}\label{1_eq_V1close2}
	\textbf Z = \big\{\{\mu + 1/\alpha -1\} \textbf I + \{2-\mu-1/\alpha\} \hat{\tilde{\textbf{A}}}\big\}^{-1} \{\mu \textbf I + (1-\mu) \hat{\tilde{\textbf{A}}} \} \textbf H.
\end{equation}

\subsubsection{Iterative Approximation} \label{4.1.3} \quad Considering that the closed-form solution is computationally inefficient because of the matrix inversion, we can use the following iterative approximation solution instead without constructing the dense inverse matrix:
\begin{equation}\label{1_eq_V1iter}
\textbf Z^{(k+1)} =\frac{1 +\alpha\mu - 2\alpha}{1+\alpha\mu -\alpha} \hat{\tilde{\textbf{A}}} \textbf{Z}^{(k)} + \frac{\alpha\mu}{1+\alpha\mu -\alpha} \textbf{H} + \frac{\alpha - \alpha \mu}{1+\alpha\mu -\alpha} \hat{\tilde{\textbf{A}}}\textbf{H},
\end{equation}
which converge to the closed-form solution in Eq. (\ref{1_eq_V1close2}) when $k \rightarrow \infty$, and with $\alpha \in (0, 2/3)$, all the coefficients are always positive.

\subsubsection{Model Design} \quad With the derived two propagation strategies in Eq. (\ref{1_eq_V1close2}) and Eq. (\ref{1_eq_V1iter}), we propose two new GNNs in both \textbf{closed} and \textbf{iter}ative forms. Note that we represent the proposed models as \textbf{GNN} with \textbf{L}ow-pass \textbf{F}iltering graph convolutional kernel (\textbf{GNN-LF}).

\textbf{GNN-LF-closed} \quad Using the closed-form propagation matrix in Eq. (\ref{1_eq_V1close2}), we define the following propagation mechanism with $\mu \in [1/2, 1), \alpha \in (0, 2/3)$ and $\textbf H = f_{\theta}(\textbf X)$:
\begin{equation}\label{1_eq_V1design1}
\begin{aligned}
	\textbf{Z} &= \textbf{PROPAGATE}(\textbf{X}; \mathcal{G}; \infty)_{LF-closed}\\
	&= \big\{\{\mu + 1/\alpha -1\} \textbf I + \{2-\mu-1/\alpha\} \hat{\tilde{\textbf{A}}}\big\}^{-1} \{\mu \textbf I + (1-\mu) \hat{\tilde{\textbf{A}}} \} \textbf H, \\
	a&nd \quad \textbf H = f_{\theta}(\textbf X).
\end{aligned}
\end{equation}

Here we first get a non-linear transformation result $\textbf{H}$ on feature \textbf{X} with an MLP network $f_{\theta}(\cdot)$, and use the designed propagation matrix $\{\{\mu + 1/\alpha -1\} \textbf I + \{2-\mu-1/\alpha\} \hat{\tilde{\textbf{A}}}\big\}^{-1}$ to propagate both $\textbf{H}$ and $\textbf{A}\textbf{H}$, then we can get the representation encoding feature information from both original and low-frequency spaces.

\textbf{GNN-LF-iter} \quad Using the iter-form propagation mechanism in Eq. (\ref{1_eq_V1iter}), we can design a deep and computationally efficient graph neural network with $\mu \in [1/2, 1), \alpha \in (0, 2/3)$:
\begin{equation}\label{1_eq_V1design2}
	\begin{aligned}
		\textbf{Z} &= \textbf{PROPAGATE}(\textbf{X}; \mathcal{G}; K)_{LF-iter}\\
		\quad &= \bigg\langle \frac{1 +\alpha\mu - 2\alpha}{1+\alpha\mu -\alpha} \hat{\tilde{\textbf{A}}} \textbf Z^{(k)} + \frac{\alpha\mu}{1+\alpha\mu -\alpha} \textbf H + \frac{\alpha - \alpha \mu}{1+\alpha\mu -\alpha} \hat{\tilde{\textbf{A}}} \textbf H\bigg\rangle_K,\\
		 \textbf Z&^{(0)} = \frac{\mu}{1+\alpha\mu -\alpha} \textbf H + \frac{1 - \mu}{1+\alpha\mu -\alpha} \hat{\tilde{\textbf{A}}} \textbf H, \quad and \quad \textbf H = f_{\theta}(\textbf X).\\
	\end{aligned}
\end{equation}
We directly use the $K$-layer output as the propagation results. This iterative propagation mechanism can be viewed as layer-wise $\hat{\tilde{\textbf{A}}}$ based neighborhood aggregation, with residual connection on feature matrix $\textbf{H}$ and filtered feature matrix $\hat{\tilde{\textbf{A}}} \textbf H$. Note that we decouple the layer-wise transformation and aggregation process like \cite{klicpera2019predict, liu2020towards}, which is beneficial to alleviate the over-smoothing problem. GNN-LF-iter and GNN-LF-closed have the following relation:

\begin{theorem} \label{GNN_LF_convergence}
	With $K\rightarrow\infty$, deep GNN-LF-iter converges to GNN-LF-closed with the same propagation result as Eq. (\ref{1_eq_V1close2}).
\end{theorem}
\begin{proof} \label{proof_GNN_LF}
	After the $K$-layer propagation using GNN-LF-iter, the corresponding expansion result can be written as:
	\begin{equation}\label{2_eq_converproof0}
		\begin{aligned}
			\textbf Z ^ {(k)} = & \bigg\{(\frac{1 +\alpha\mu - 2\alpha}{1+\alpha\mu -\alpha})^k \hat{\tilde{\textbf{A}}}^k + \alpha \sum_{i=0}^{k-1}(\frac{1 +\alpha\mu - 2\alpha}{1+\alpha\mu -\alpha})^i \hat{\tilde{\textbf{A}}}^i\bigg\}\bigg\{\frac{\mu}{1+\alpha\mu -\alpha} \textbf H \\
			&+ \frac{1 - \mu}{1+\alpha\mu -\alpha} \hat{\tilde{\textbf{A}}}\textbf H\bigg\},
		\end{aligned}
	\end{equation}
	where $\mu \in [1/2, 1)$, $\alpha \in (0, 2/3)$ and $|\frac{1 +\alpha\mu - 2\alpha}{1+\alpha\mu -\alpha}|< 1 $. When $k \rightarrow \infty$, the left term tends to 0 and the right term becomes a geometric series. The series converges since $\hat{\tilde{\textbf{A}}}$ has absolute eigenvalues bounded by 1, then Eq. (\ref{2_eq_converproof0}) can be rewritten as:
	\begin{equation}\label{2_eq_converproof6}
		\textbf Z^{(\infty)} = \big\{\{\mu + 1/\alpha -1\} \textbf I + \{2-\mu-1/\alpha\} \hat{\tilde{\textbf{A}}}\big\}^{-1} \{\mu \textbf I + (1-\mu) \hat{\tilde{\textbf{A}}} \} \textbf H,
	\end{equation}
	which exactly is the equation for calculating GNN-LF-closed.
\end{proof}

The training of GNN-LF is also the same with other GNNs. For example, it evaluates the cross-entropy loss over all labeled examples for semi-supervised multi-class classification task.

\subsection{GNN with High-pass Filtering Kernel}
\subsubsection{Objective Function} \quad Similar with GNN-LF, we now consider preserving the similarity of \textbf{H} and \textbf{Z} in both original and high-pass filtering spaces. For neatness of the subsequent analysis, we choose the following objective:
\begin{theorem}
	With $\textbf{F}_1 = \textbf{F}_2 =  \{\textbf I + \beta \tilde{\textbf{L}} \}^{1/2}, \beta \in (0, \infty)$, $\zeta = 1$ and $\xi = 1/\alpha - 1, \alpha \in (0, 1]$ in Eq. (\ref{total_framework}), the propagation process considering flexible high-pass convolutional kernel on feature is:
	\begin{equation}\label{1_eq_V1obj2}
		\mathcal{O} = \min \limits_{\textbf{Z}} \big\{\big\| \{\textbf I + \beta \tilde{\textbf{L}} \}^{1/2} (\textbf Z - \textbf H) \big\|_F^2 + \xi tr(\textbf Z^{T} \tilde{\textbf{L}} \textbf Z)\big\}, \\
	\end{equation}
	where $\textbf{H}= f_{\theta}(\textbf X)$.
\end{theorem}
Analogously, $\beta$ is also a balance coefficient, and we set  $\beta \in (0, \infty)$ so that $\textbf I + \beta \tilde{\textbf{L}} = \textbf{\textit{V}}^{*} \pmb{\Lambda}^{*} \textbf{\textit{V}}^{*T}$ is a symmetric and positive semi-definite matrix and the matrix $\{\textbf I + \beta \tilde{\textbf{L}} \}^{1/2} = \textbf{\textit{V}}^{*} \pmb{\Lambda}^{*1/2} \textbf{\textit{V}}^{*T}$ has a filtering behavior similar to that of $\{\textbf I + \beta \tilde{\textbf{L}}\}$. As can be seen in Eq. (\ref{1_eq_V1obj2}), by adjusting the balance coefficient $\beta$, the designed objectives can flexibly constrain the similarity of $\textbf{Z}$ and $\textbf{H}$ in both original and high-frequency spaces.

\subsubsection{Closed Solution}\label{4.2.2} \quad We calculate the closed-form solution as:
\begin{equation}\label{2_eq_V2close1}
\textbf{Z} =\big\{\textbf I + (\beta + 1/\alpha -1) \tilde{\textbf{L}} \big\}^{-1} \{\textbf I + \beta \tilde{\textbf{L}} \}\textbf H,
\end{equation}
it also can be rewritten as:
\begin{equation}\label{2_eq_V2close2}
	\textbf{Z} =\big\{(\beta + 1/\alpha) \textbf I + (1 - \beta - 1/\alpha) \hat{\tilde{\textbf{A}}} \big\}^{-1} \{\textbf I + \beta \tilde{\textbf{L}} \}\textbf H.
\end{equation}

\subsubsection{Iterative Approximation}\label{4.2.3} \quad Considering it is inefficient to calculate the inverse matrix, we give the following iterative approximation solution without constructing the dense inverse matrix:
\begin{equation}\label{2_eq_V2iter}
\textbf Z^{(k+1)} = \frac{\alpha \beta - \alpha + 1}{\alpha \beta + 1} \hat{\tilde{\textbf{A}}} \textbf Z^{(k)} + \frac{\alpha}{\alpha \beta + 1} \textbf H + \frac{\alpha \beta}{\alpha \beta + 1} \tilde{\textbf{L}} \textbf H.
\end{equation}

\subsubsection{Model Design} \quad With the derived two propagation strategies in Eq. (\ref{2_eq_V2close2}) and in Eq. (\ref{2_eq_V2iter}), we propose two new GNNs in both \textbf{closed} and \textbf{iter}ative forms. Similarly, we use \textbf{GNN-HF} to denote \textbf{GNN} with \textbf{H}igh-pass \textbf{F}iltering graph convolutional kernels.

\textbf{GNN-HF-closed} \quad Using the closed-form propagation matrix in Eq. (\ref{2_eq_V2close2}), we define the following new graph neural networks with closed-form propagation mechanism:
\begin{equation}\label{1_eq_V2design1}
	\begin{aligned}
	\textbf{Z} &= \textbf{PROPAGATE}(\textbf{X}; \mathcal{G}; \infty)_{HF-closed}\\
	&=\big\{(\beta + 1/\alpha) \textbf I + (1 - \beta - 1/\alpha) \hat{\tilde{\textbf{A}}} \big\}^{-1} \{\textbf I + \beta \tilde{\textbf{L}} \}\textbf H,\\
	a&nd \quad \textbf H = f_{\theta}(\textbf X).
	\end{aligned}
\end{equation}
Note that $\beta \in (0, \infty)$ and $\alpha \in (0, 1]$. By applying the propagation matrix $\{(\beta + 1/\alpha) \textbf I + (1 - \beta - 1/\alpha) \hat{\tilde{\textbf{A}}}\}^{-1}$ directly on both $\textbf{H}$ and $ \tilde{\textbf{L}} \textbf H$ matrix, then we can get the representation encoding feature information from both original and high-frequency spaces.

\textbf{GNN-HF-iter} \quad Using the iterative propagation mechanism in Section \ref{4.2.3}, we have a deep and computationally efficient graph neural networks with $\beta \in (0, \infty)$ and $\alpha \in (0, 1]$.
\begin{equation}\label{1_eq_V2design2}
	\begin{aligned}
		\textbf{Z} &= \textbf{PROPAGATE}(\textbf{X}; \mathcal{G}; K)_{HF-iter}\\
		&= \bigg\langle \frac{\alpha \beta - \alpha + 1}{\alpha \beta + 1} \hat{\tilde{\textbf{A}}} \textbf Z^{(k)} + \frac{\alpha}{\alpha \beta + 1} \textbf H + \frac{\alpha \beta}{\alpha \beta + 1} \tilde{\textbf{L}} \textbf H \bigg\rangle_K,\\
		\textbf Z&^ {(0)} = \frac{1}{\alpha \beta + 1} \textbf H + \frac{\beta}{\alpha \beta + 1} \tilde{\textbf{L}} \textbf H, \quad and \quad \textbf H = f_{\theta}(\textbf X).\\
	\end{aligned}
\end{equation}
We directly use the $K$-layer output as the propagation results. Similarly, this iterative propagation mechanism can be viewed as layer-wise $\hat{\tilde{\textbf{A}}}$ based neighborhood aggregation, and residual connection on both feature matrix $\textbf{H}$ and high-frequency filtered feature matrix $\tilde{\textbf{L}} \textbf H$. And we also decouple the layer-wise transformation and aggregation process during propagation. GNN-HF-iter and GNN-HF-closed have the following relation:
\begin{theorem} \label{GNN_HF_convergence}
When $K\rightarrow\infty$, deep GNN-HF-iter converges to GNN-HF-closed with the same propagation result as Eq. (\ref{2_eq_V2close2}).
\end{theorem}
\begin{proof}
	Please refer to Appendix \ref{sec::conver_HF}.
\end{proof}

\section{Spectral Expressive Power Analysis}\label{sec::Expressive}
In this section, we study several propagation mechanisms in graph spectral domain to examine their expressive power. A polynomial filter of order $K$ on a graph signal $\textbf x $ is defined as $\big(\sum_{k=0}^{K} \theta_k \tilde{\textbf{L}}^k\big) \textbf x$, where $\theta_k$ is the corresponding polynomial coefficients. Similar with \cite{chen2020simple}, we also assume that the graph signal \textbf{x} is non-negative and \textbf{x} can be converted into the input signal \textbf{H} under linear transformation. We aim to compare the polynomial coefficients $\theta_k$ for different GNNs and show that GNN-LF/HF with $K$ order flexible polynomial filter coefficients have better spectral expressive power. For concision, we mainly analyze the filter coefficients of GNN-LF, and the spectral analysis of SGC, PPNP, GNN-HF is in Appendix \ref{sec::Polynomial Filter}.

\subsection{Filter Coefficient Analysis} \label{filter}
\textbf{Analysis of GNN-LF.} \quad From the analysis in Theorem \ref{GNN_LF_convergence}, we have the expanded propagation result of GNN-LF-iter in Eq. (\ref{2_eq_converproof0}), which has been proved to converge to the propagation result of GNN-LF-closed with $K \rightarrow \infty$. Taking this propagation result for analysis, we have the following filtering expression when $K \rightarrow \infty$:
\begin{equation}\label{4_eq_V1filter1}
	\begin{aligned}
		\textbf z ^ {(K)} &= \Big\{\alpha \sum_{i=0}^{K-1}\big(\frac{1 +\alpha\mu - 2\alpha}{1+\alpha\mu -\alpha}\big)^i \hat{\tilde{\textbf{A}}}^i\Big\}\Big\{\frac{\mu}{1+\alpha\mu -\alpha} \textbf x + \frac{1 - \mu}{1+\alpha\mu -\alpha} \hat{\tilde{\textbf{A}}}\textbf x\Big\} \\
		&= \frac{\alpha\mu}{1+\alpha\mu -\alpha} \Big\{\sum_{i=0}^{K-1}\big(\frac{1 +\alpha\mu - 2\alpha}{1+\alpha\mu -\alpha}\big)^i \hat{\tilde{\textbf{A}}}^i \Big\} \textbf x + \frac{\alpha - \alpha\mu}{1+\alpha\mu -2\alpha} \bm{\cdot} \\
		&\Big\{\sum_{i=1}^{K}\big(\frac{1 +\alpha\mu - 2\alpha}{1+\alpha\mu -\alpha}\big)^i \hat{\tilde{\textbf{A}}}^i \Big\}\textbf x \\
		&= \frac{\alpha\mu}{1+\alpha\mu -\alpha} \Big\{\sum_{i=0}^{K-1}\big(\frac{1 +\alpha\mu - 2\alpha}{1+\alpha\mu -\alpha}\big)^i \big(\textbf I - \tilde{\textbf{L}}\big)^i \Big\} \textbf x + \frac{\alpha - \alpha\mu}{1+\alpha\mu -2\alpha} \bm{\cdot} \\
		&\Big\{\sum_{i=1}^{K}\big(\frac{1 +\alpha\mu - 2\alpha}{1+\alpha\mu -\alpha}\big)^i \big(\textbf I - \tilde{\textbf{L}}\big)^i \Big\}\textbf x.
	\end{aligned}
\end{equation}
Expand the above equation, then the filter coefficients on $\tilde{\textbf{L}}^k (k \in [0, K])$ can be summarized into the following forms:
\begin{itemize}
	\item[1)] \textit{Filter coefficients for} $\tilde{\textbf{L}}^0$:
\begin{equation}\label{4_eq_V1filter2}
	\begin{aligned}
		&\theta_{0} = \frac{\alpha \mu (1 + \alpha \mu -2 \alpha)}{(1 + \alpha \mu -\alpha)^2} + \frac{(\alpha - \alpha\mu)(1 + \alpha\mu -2\alpha)^{K-1}}{(1 + \alpha\mu -\alpha)^{K}} + \sum\limits_{j=1}^{K-1} \delta_j\binom{j}{0}, \\
		&\delta_j = \big\{\frac{\alpha \mu}{1 + \alpha \mu - \alpha} + \frac{\alpha - \alpha\mu}{1 + \alpha\mu -2\alpha}\big\} \big(\frac{1 + \alpha \mu -2 \alpha}{1 + \alpha \mu -\alpha}\big)^j.
	\end{aligned}
\end{equation}
	\item[2)] \textit{Filter coefficients for} $\tilde{\textbf{L}}^k, k \in [1, K-1]$:
\begin{equation}\label{4_eq_V1filter3}
	\begin{aligned}
		\theta_{k} &= \sum\limits_{j=k}^{K} \delta_j(-1)^k\binom{j}{k}, \\
		\delta_j &= \big\{\frac{\alpha \mu}{1 + \alpha \mu - \alpha} + \frac{\alpha - \alpha\mu}{1 + \alpha\mu -2\alpha}\big\} \big(\frac{1 + \alpha \mu -2 \alpha}{1 + \alpha \mu -\alpha}\big)^j.
	\end{aligned}
\end{equation}
	\item[3)] \textit{Filter coefficients for} $\tilde{\textbf{L}}^K$:
\begin{equation}\label{4_eq_V1filter4}
	\theta_{K} = \frac{(\alpha - \alpha\mu)(1 + \alpha\mu -2\alpha)^{K-1}}{(1 + \alpha\mu -\alpha)^{K}}(-1)^K\binom{K}{K}.
\end{equation}
\end{itemize}
From the above analysis result on GNN-LF, we find the expression forms of filter coefficients depend on different $k$ and are determined by two adjustable factors $\alpha$ and $\mu$, which improve the expressive power of the spectral filters and further alleviate the over-smoothing problem.

\textbf{Analysis of SGC/PPNP/GNN-HF.} \quad Note that the analysis of GNN-HF is similar with that of GNN-LF, for concision, we show them in Appendix \ref{sec::Polynomial Filter}.

\subsection{Discussion on Expressive Power}
As \cite{xu2018representation, liu2020towards} point out, the reason for the over-smoothing problem is that typical GCN converges to the limit distribution of random walk which is isolated from the input feature and makes node representations inseparable as the number of layer increases. \cite{chen2020simple} also gives another understanding from the view of polynomial filtering coefficient and points out that flexible and arbitrary filter coefficients are essential for preventing over-smoothing.

From the filter coefficients shown in Section \ref{filter} and Appendix \ref{sec::Polynomial Filter}, we can find that: 1) SGC or $K$-layer graph convolutional operations have fixed constant filtering coefficients, which limit the expressive power and further lead to over-smoothing. 2) PPNP has a better filtering expressive ability against SGC (GCN) since the filter coefficients of the order $k$ is changeable along with the factor $\alpha$. 3) Comparing with PPNP and SGC (GCN), GNN-LF/HF are more expressive under the influence of adjustable factors $\alpha$, $\mu$  or $\beta$, which increase the ability to fit arbitrary coefficients of polynomial filter, and help GNN-LF/HF to alleviate the over-smoothing problem.

From the limit distributions of PPNP \cite{klicpera2019predict}, GNN-LF in Eq. (\ref{2_eq_converproof6}), GNN-HF in Eq. (\ref{3_eq_converproof6}), we can also find that all of them converge to a distribution carrying information from both input feature and network structure. This property additionally helps to reduce the effects of over-smoothing on PPNP/GNN-LF/GNN-HF even if the number of layers goes to infinity.

\section{Experiments}\label{sec::Experiments}
\begin{table}[!t]
	\centering
	\caption{The statistics of the datasets}
	\label{dataset}
	\renewcommand\arraystretch{1.12}
	\setlength{\tabcolsep}{0.6mm}{
		\begin{tabular}{lcrrrc}
			\hline
			\textbf{Dataset}&\textbf{Classes}&\textbf{Nodes}&\textbf{Edges}&\textbf{Features}&\textbf{Train}/\textbf{Val}/\textbf{Test}\\
			\hline
			\textbf{Cora}&7&2708&5429&1433&140/500/1000\\
			\textbf{Citeseer}&6&3327&4732&3703&120/500/1000\\
			\textbf{Pubmed}&3&19717&44338&500&60/500/1000\\
			\textbf{ACM}&3&3025&13128&1870&60/500/1000\\
			\textbf{Wiki-CS}&10&11701&216123&300&200/500/1000\\
			\textbf{MS Academic}&15&18333&81894&6805&300/500/1000\\
			\hline
	\end{tabular}}
\end{table}

\subsection{Experimental Setup}
\textbf{Dataset.} \quad To evaluate the effectiveness of our proposed GNN-LF/HF, we conduct experiments on six benchmark datasets in Table \ref{dataset}. 1) \textbf{Cora, Citeseer, Pubmed} \cite{kipf2017semi}: Three standard citation networks where nodes represent documents, edges are citation links and features are the bag-of-words representation of the document. 2) \textbf{ACM} \cite{wang2019heterogeneous}: Nodes represent papers and there is an edge if two paper have same authors. Features are the bag-of-words representations of paper keywords. The three classes are \textit{Database}, \textit{Wireless Communication}, \textit{DataMining}.  3) \textbf{Wiki-CS} \cite{mernyei2020wiki}: A dataset derived from Wikipedia, in which nodes represent CS articles, edges are hyperlinks and different classes mean different branches of the files. 4) \textbf{MS Academic} \cite{klicpera2019predict}: A co-authorship Microsoft Academic Graph, where nodes represent authors, edges are co-authorships and node features represent keywords from authors' papers.

\begin{table*}[!h]
	\centering
	\renewcommand\arraystretch{1.2}
	\caption{Node classification results (\%). We show the average accuracy with uncertainties showing the 95\% confidence level calculated by boot-strapping. Bold and underline are used to show the best and the runner-up results.}
	\vspace{5pt}
	\label{node classification}
	\resizebox{0.9\textwidth}{!}{
		\begin{tabular}{c|c|c|c|c|c|c}
			\hline
			\multirow{2}{*}{\textbf{Model}}&
			\multicolumn{6}{c}{\textbf{Dataset}}\\
			\cline{2-7}
			&\textbf{Cora}&\textbf{Citeseer}&\textbf{Pubmed}&\textbf{ACM}&\textbf{Wiki-CS}&\textbf{MS Academic}\\
			\hline
			\textbf{MLP}             &57.79$\pm$0.11&61.20$\pm$0.08&73.23$\pm$0.05&77.39$\pm$0.11&65.66$\pm$0.20&87.79$\pm$0.42\\
			\textbf{LP}             &71.50$\pm$0.00&50.80$\pm$0.00&72.70$\pm$0.00&63.30$\pm$0.00&34.90$\pm$0.00&74.10$\pm$0.00\\
			\textbf{ChebNet}        &79.92$\pm$0.18&70.90$\pm$0.37&76.98$\pm$0.16&79.53$\pm$1.24&63.24$\pm$1.43&90.76$\pm$0.73\\
			\textbf{GAT}            &82.48$\pm$0.31&72.08$\pm$0.41&79.08$\pm$0.22&88.24$\pm$0.38&74.27$\pm$0.63&91.58$\pm$0.25\\
			\textbf{GraphSAGE}      &82.14$\pm$0.25&71.80$\pm$0.36&79.20$\pm$0.27&87.57$\pm$0.65&73.17$\pm$0.41&91.53$\pm$0.15\\
			\textbf{IncepGCN}       &81.94$\pm$0.94&69.66$\pm$0.29&78.88$\pm$0.35&87.75$\pm$0.61&60.54$\pm$1.06&75.45$\pm$0.49\\
			\hline
			\textbf{GCN}            &82.41$\pm$0.25&70.72$\pm$0.36&79.40$\pm$0.15&88.38$\pm$0.51&71.97$\pm$0.51&92.17$\pm$0.11\\
			\textbf{SGC}            &81.90$\pm$0.23&\underline{72.21$\pm$0.22}&78.30$\pm$0.14&87.56$\pm$0.34&72.43$\pm$0.28&88.35$\pm$0.36\\
			\textbf{PPNP}           &83.34$\pm$0.20&71.73$\pm$0.30&80.06$\pm$0.20&89.12$\pm$0.17&74.53$\pm$0.36&92.27$\pm$0.23\\
			\textbf{APPNP}          &83.32$\pm$0.42&71.67$\pm$0.48&80.05$\pm$0.27&89.04$\pm$0.21&74.30$\pm$0.50&92.25$\pm$0.18\\
			\textbf{JKNet}          &81.19$\pm$0.49&70.69$\pm$0.88&78.60$\pm$0.25&88.11$\pm$0.36&60.90$\pm$0.92&87.26$\pm$0.23\\
			\hline
			\textbf{GNN-LF-closed}&83.70$\pm$0.14&71.98$\pm$0.33&80.34$\pm$0.18&89.43$\pm$0.20&\textbf{75.50$\pm$0.56}&\textbf{92.79$\pm$0.15}\\
			\textbf{GNN-LF-iter}  &83.53$\pm$0.24&71.92$\pm$0.24&80.33$\pm$0.20&89.37$\pm$0.40&\underline{75.35$\pm$0.24}&\underline{92.69$\pm$0.20}\\
			\hline
			\textbf{GNN-HF-closed}&\textbf{83.96$\pm$0.22}&\textbf{72.30$\pm$0.28}&\underline{80.41$\pm$0.25}&\underline{89.46$\pm$0.30}&74.92$\pm$0.45&92.47$\pm$0.23\\
			\textbf{GNN-HF-iter}  &\underline{83.79$\pm$0.29}&72.03$\pm$0.36&\textbf{80.54$\pm$0.25}&\textbf{89.59$\pm$0.31}&74.90$\pm$0.37&92.51$\pm$0.16\\
			\hline
	\end{tabular}}
\end{table*}

\textbf{Baselines.} \quad We evaluate the performance of GNN-LF/HF by comparing it with several baselines. 1) Traditional graph learning methods: MLP \cite{pal1992multilayer}, LP \cite{zhu2005semi}. 2) Spectral methods: ChebNet \cite{defferrard2016convolutional}, GCN \cite{kipf2017semi}; 3) Spatial methods: SGC \cite{wu2019simplifying}, GAT \cite{velickovic2018graph}, GraphSAGE \cite{hamilton2017inductive}, PPNP \cite{klicpera2019predict}. 4) Deep GNN methods: JKNet \cite{xu2018representation}, APPNP \cite{klicpera2019predict}, IncepGCN \cite{rong2020dropedge}.

\textbf{Settings.}\quad We implement GNN-LF/HF based on Pytorch \cite{paszke2017automatic}. To ensure fair comparisons, we fix the hidden size as 64 for all models. We apply $L_2$ regularization on the first layer parameter weights, with coefficients of 5e-3 on all datasets except 5e-4 for Wiki-CS.  We set the learning rate $lr=0.01$ for the other datasets except $lr=0.03$ for Wiki-CS, and set dropout rate $d = 0.5$. We use the validation set for early stopping with a patience of 100 epochs. We fix 10 propagation depth for the two iterative version of GNN-LF/HF. Note that for APPNP and PPNP, all the settings are consistent with the above descriptions. As for ChebNet, GCN, GAT, SGC and GraphSAGE, we use the implementations of DGL \cite{wang2019dgl}\footnote{https://github.com/dmlc/dgl}. For JKNet and IncepGCN, we use the implementation in \cite{rong2020dropedge}\footnote{https://github.com/DropEdge/DropEdge}. We try to turn all hyperparameters reasonably to get the best performance, so some models even achieve better results than original reports. For JKNet, IncepGCN and SGC, we choose the best results of them with no more than 10 propagation depth. We conduct 10 runs on all datasets with the fixed training/validation/test split, where 20 nodes per class are used for training and 500/1000 nodes are used for val/test. For cora/citeseer/pubmed datasets, we follow the dataset splits in \cite{yang2016revisiting}.

\subsection{Node Classification}
We evaluate the effectiveness of GNN-LF/HF against several state-of-the-art baselines on semi-supervised node classification task. We use accuracy (ACC) metric for evaluation, and report the average ACC with uncertainties showing the 95\% confidence level calculated by bootstrapping in Table \ref{node classification}. We have the following observations:

\begin{figure*}[t]
	\centering
	\subfigure[\textbf{Cora}]{
		\includegraphics[width=0.31\linewidth]{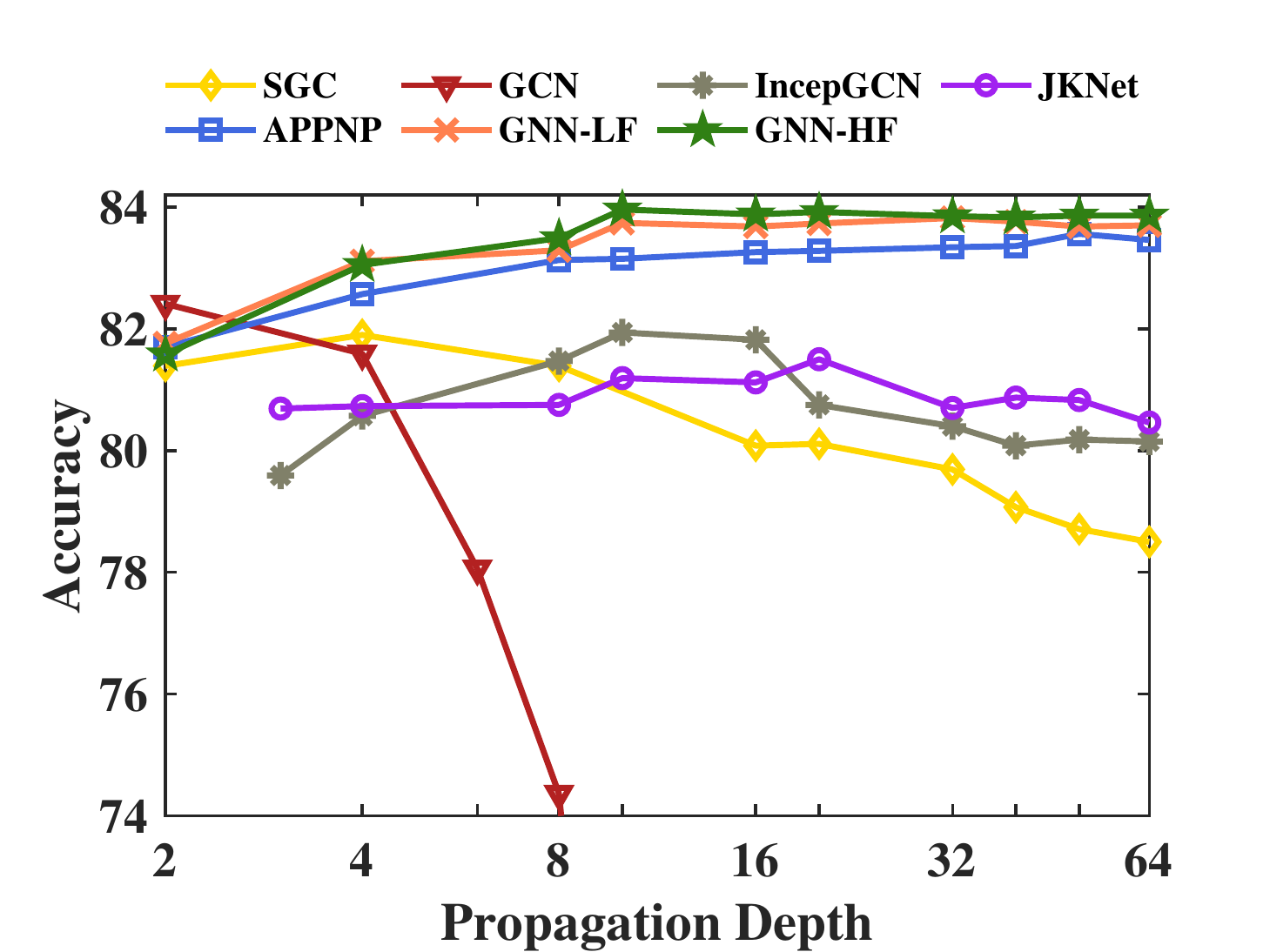}
	}
	\subfigure[\textbf{Citeseer}]{
		\includegraphics[width=0.31\linewidth]{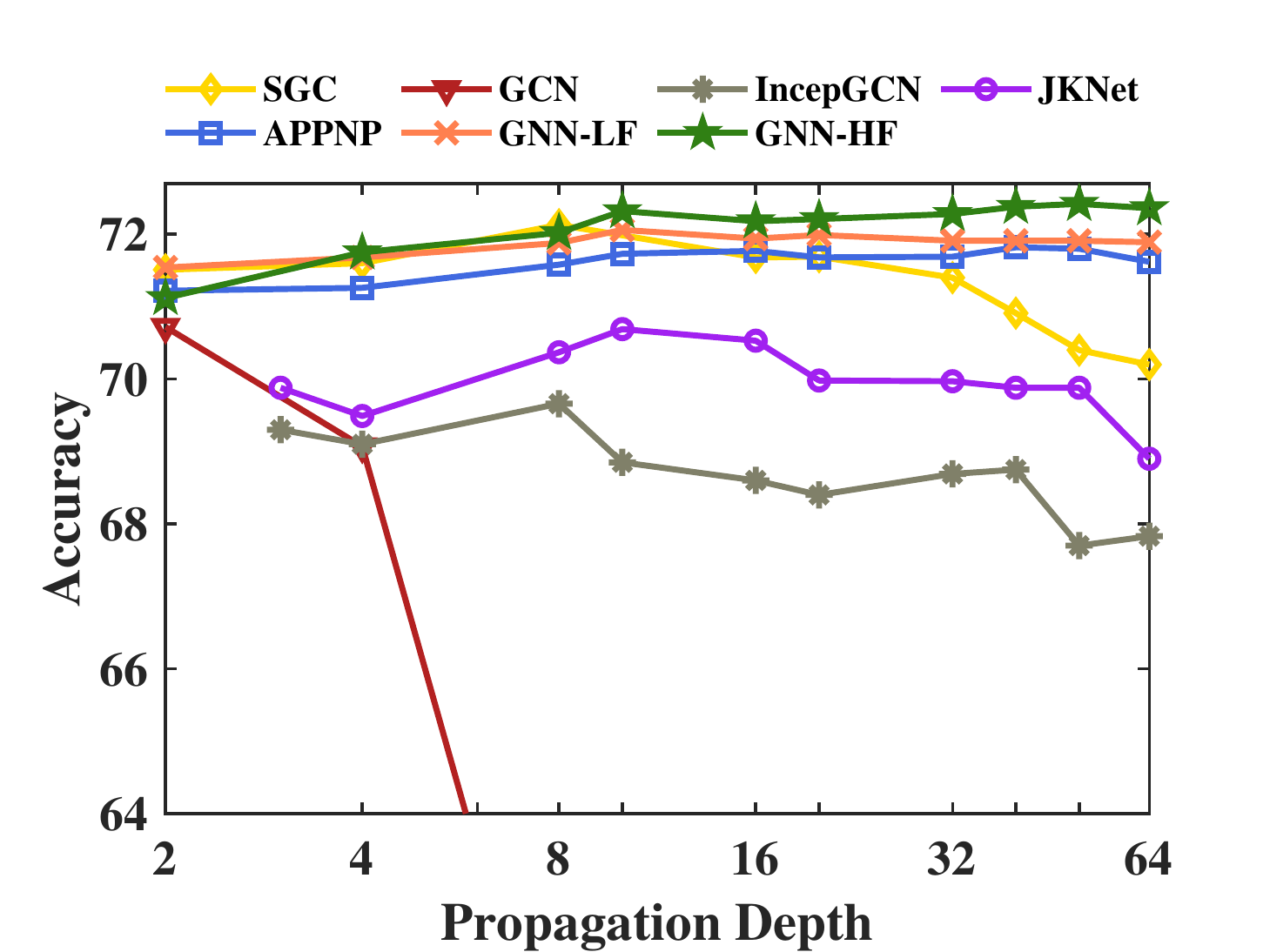}
	}
	\subfigure[\textbf{Pubmed}]{
		\includegraphics[width=0.31\linewidth]{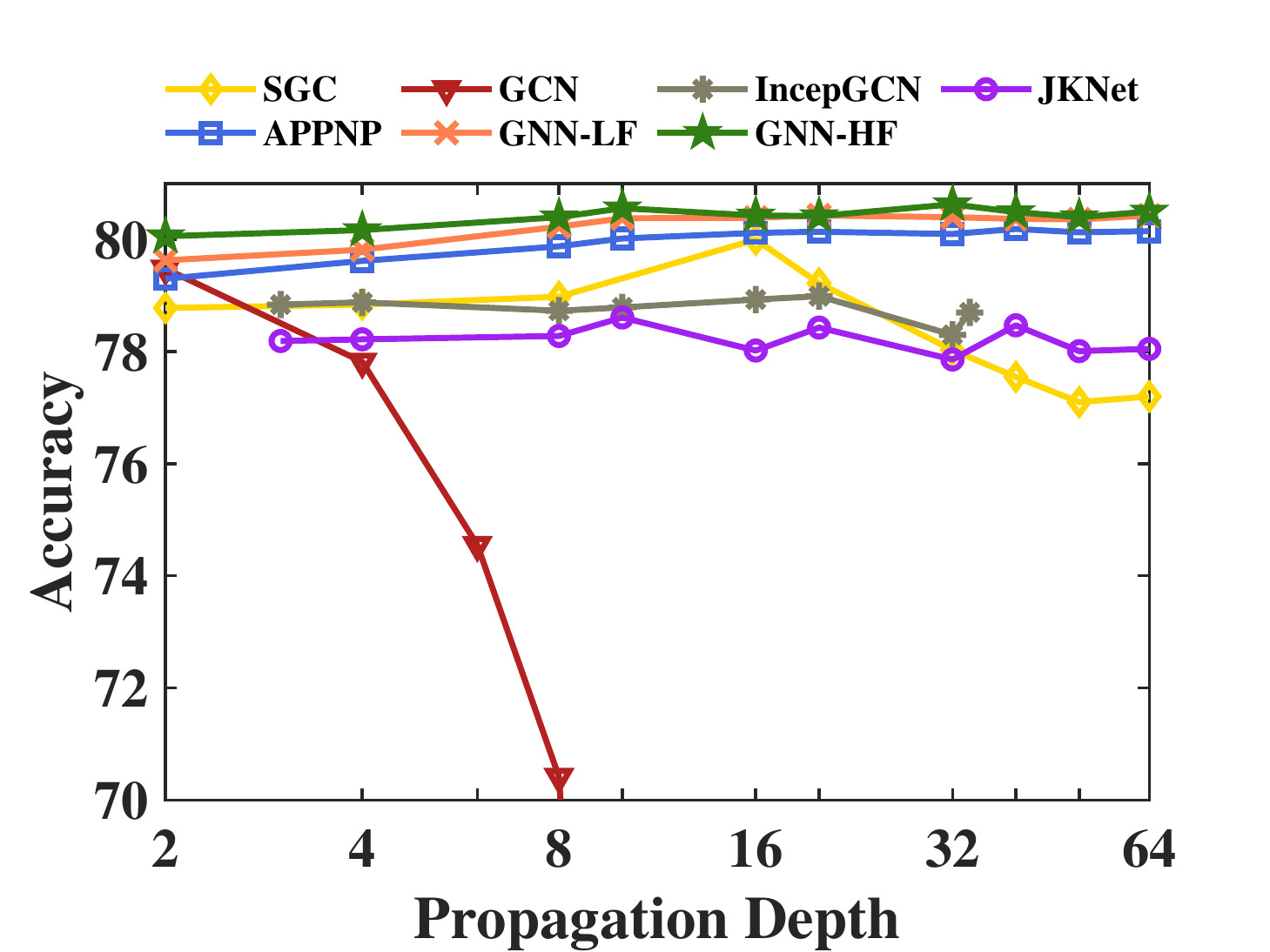}
	}
	\caption{Analysis of propagation depth.}
	\label{layer}
\end{figure*}

1) GNN-LF and GNN-HF consistently outperform all the state-of-the-art baselines on all datasets. The best and the runner-up results are always achieved by GNN-LF/HF, which demonstrates the effectiveness of our proposed model. From the perspective of the unified objective framework, it is easy to check that GNN-LF/HF not only keep the representation same with the original features, but also consider capturing their similarities based on low-frequency or high-frequency information. These two relations are balanced so as to extract more meaningful signals and thus perform better.

2) From the results of the closed and iterative versions of GNN-LF/HF, we can see that using 10 propagation depth for GNN-LF-iter/GNN-HF-iter is able to effectively approximate the GNN-LF-closed/GNN-HF-closed. As for performance comparisons between GNN-LF and GNN-HF, we find that it is hard to determine which is the best, since which filter works better may depend on the characteristic of different datasets. But in summary, flexibly and comprehensively considering multiple information in a GNN model can always achieve satisfactory results on different networks.

3) In addition, PPNP/APPNP always perform better than GCN/SGC since their objective also considers a fitting term to help find important information from features during propagation. On the other hand, APPNP outperforms JKNet mainly because that its propagation process takes full advantage of the original features and APPNP even decouples the layer-wise non-linear transformation operations \cite{liu2020towards} without suffering from performance degradation. Actually, the above differences of models and explanations for results can be easily drawn from our unified framework.

\begin{figure*}[htbp]
	\centering
	\subfigure[\textbf{GNN-LF}]{
		\includegraphics[width=0.24\linewidth]{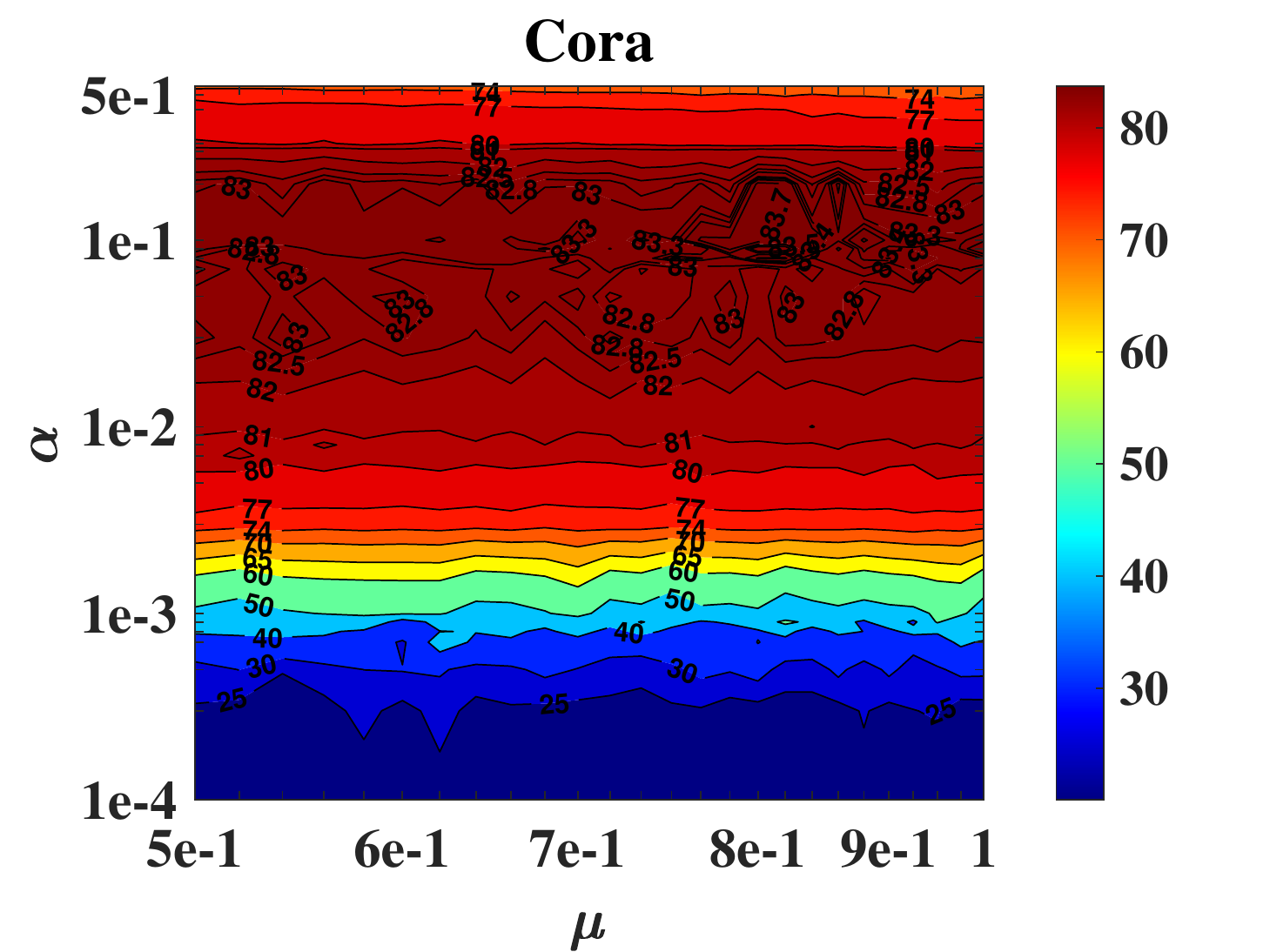}
		\includegraphics[width=0.24\linewidth]{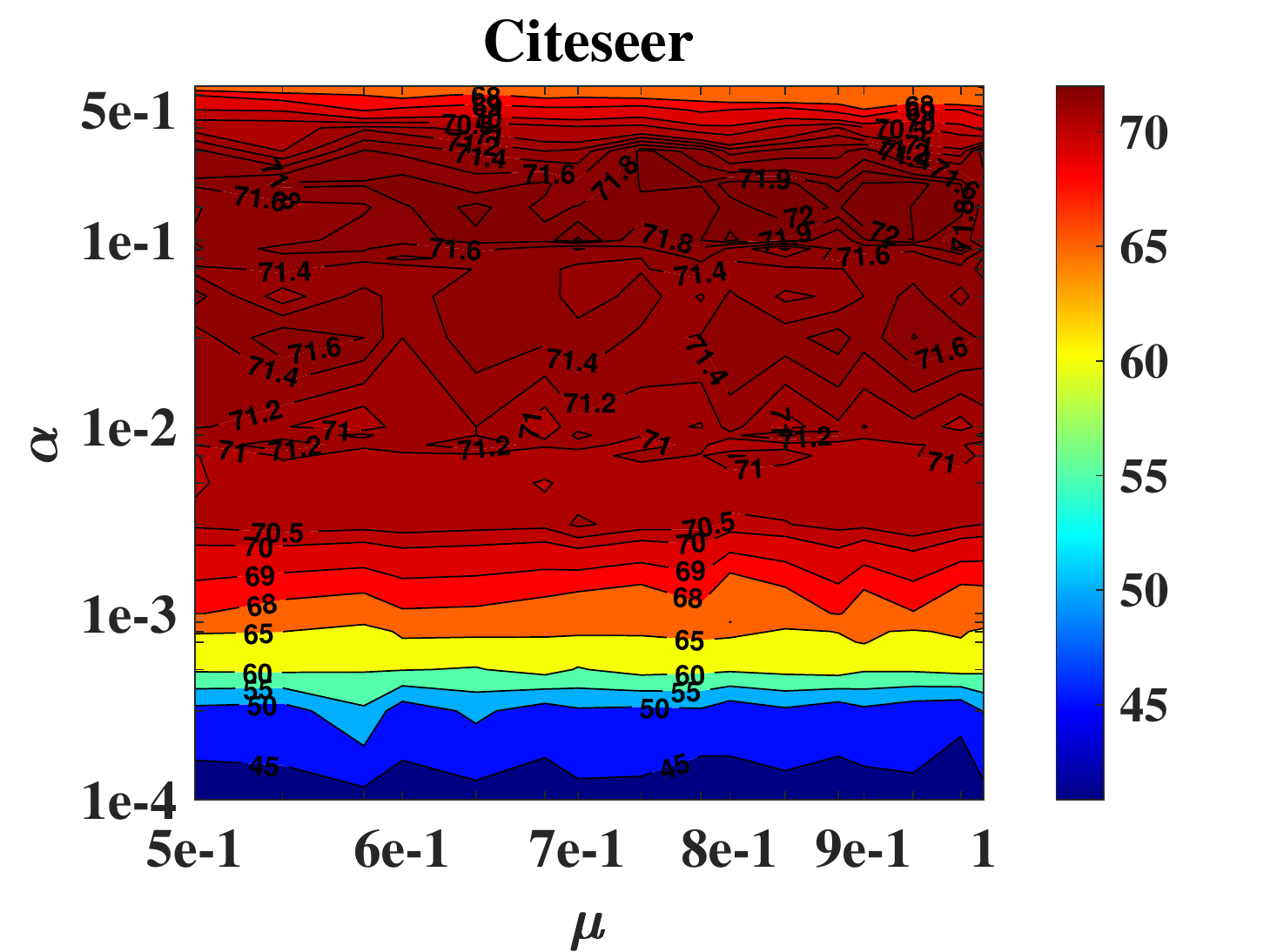}
	}
	\subfigure[\textbf{GNN-HF}]{
		\includegraphics[width=0.24\linewidth]{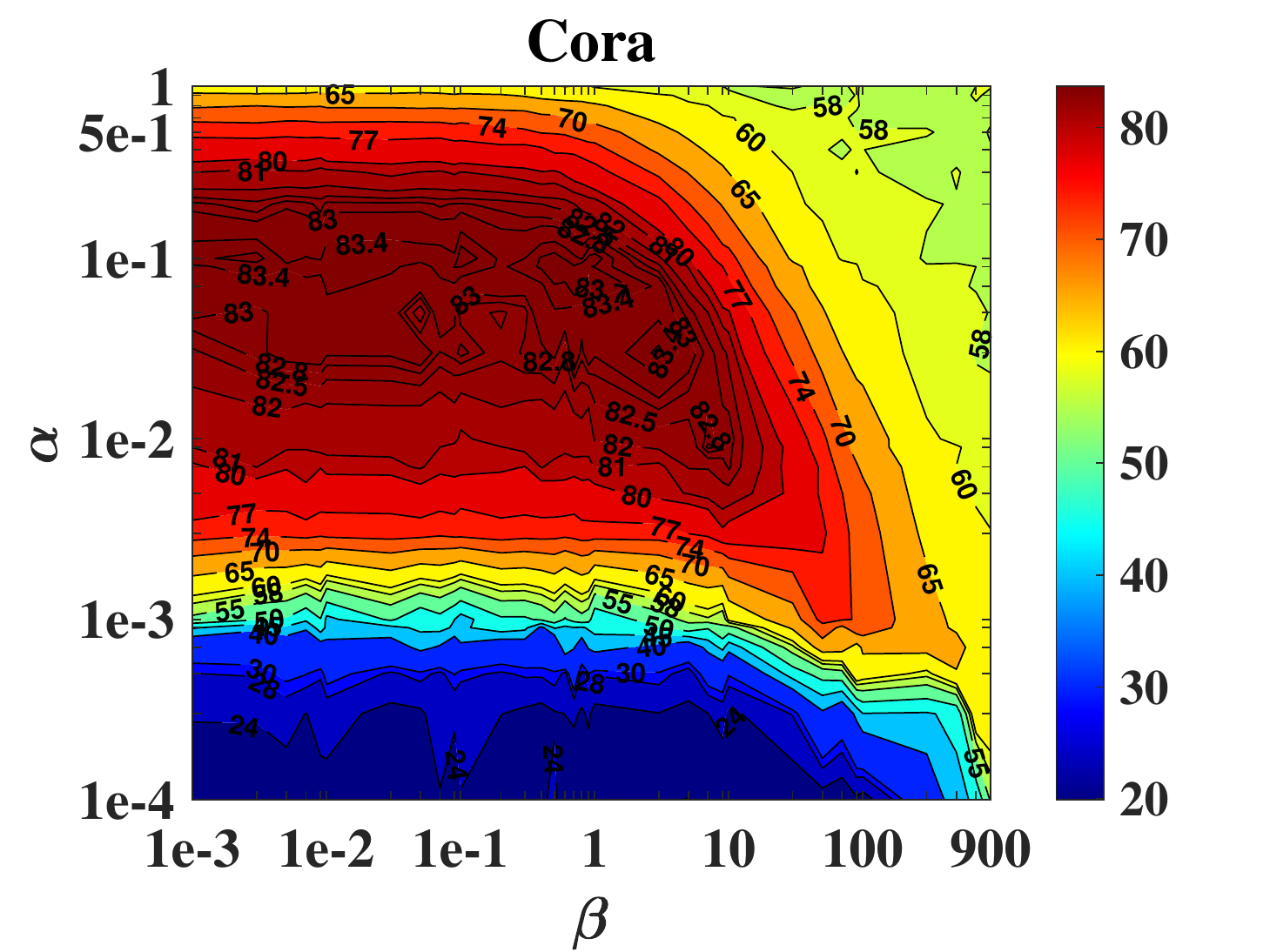}
		\includegraphics[width=0.24\linewidth]{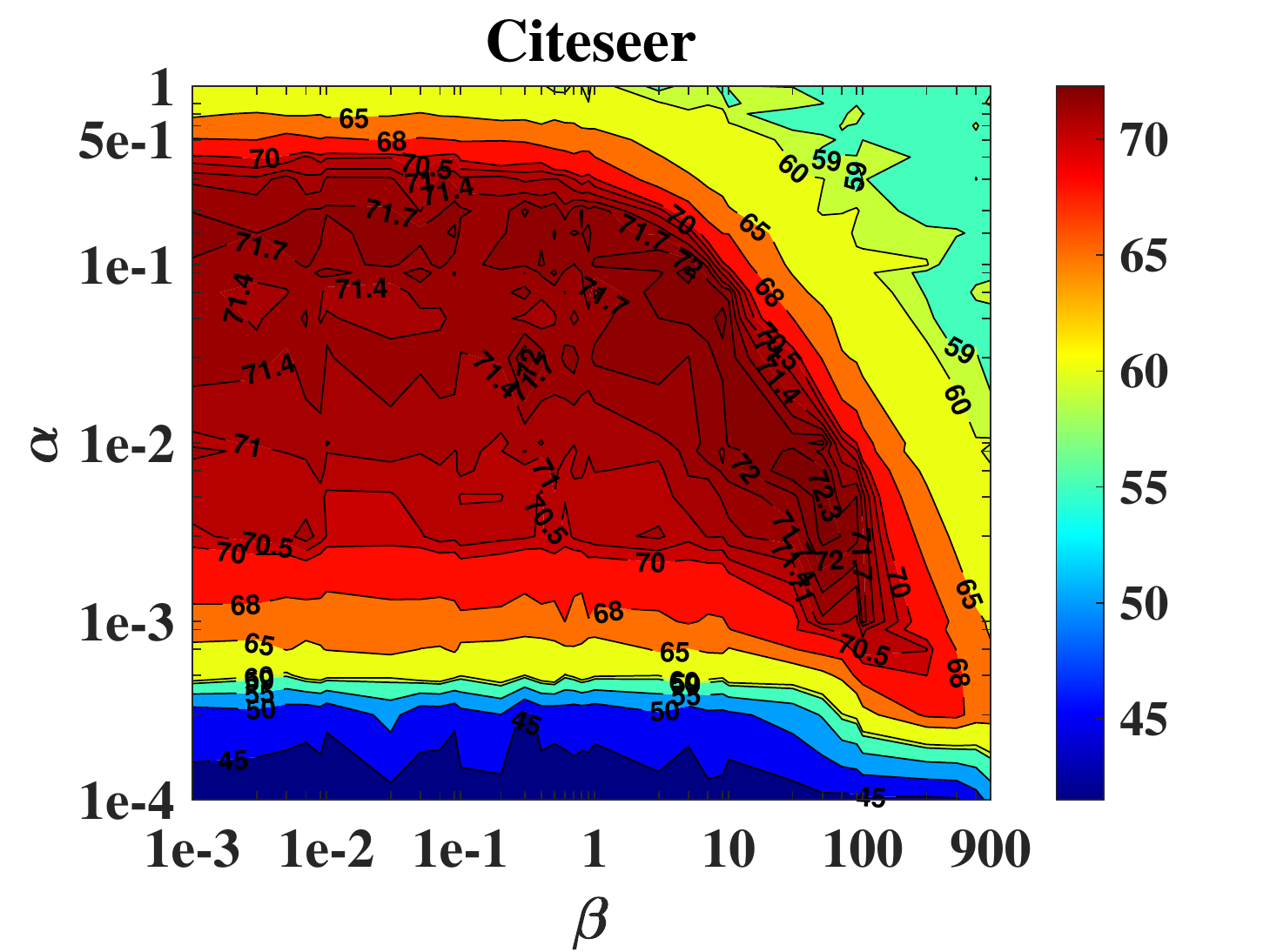}
	}
	\caption{Model analysis of \textbf{GNN-LF/HF} on Cora and Citeseer.}
	\label{parameter_cora/cite}
\end{figure*}

\begin{figure}[htbp]
	\centering
	\subfigure{
		\includegraphics[width=0.45\linewidth]{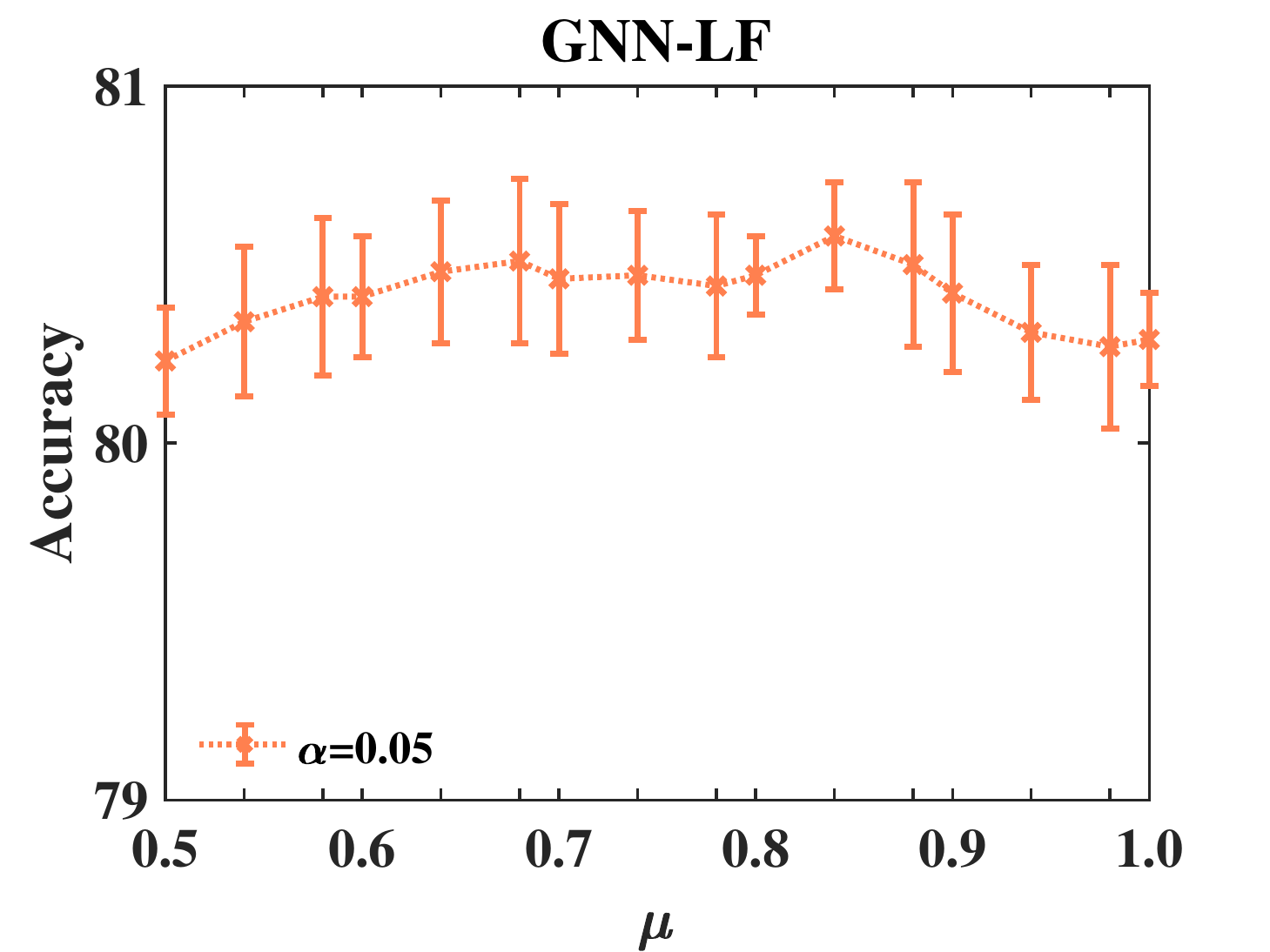}
		\includegraphics[width=0.45\linewidth]{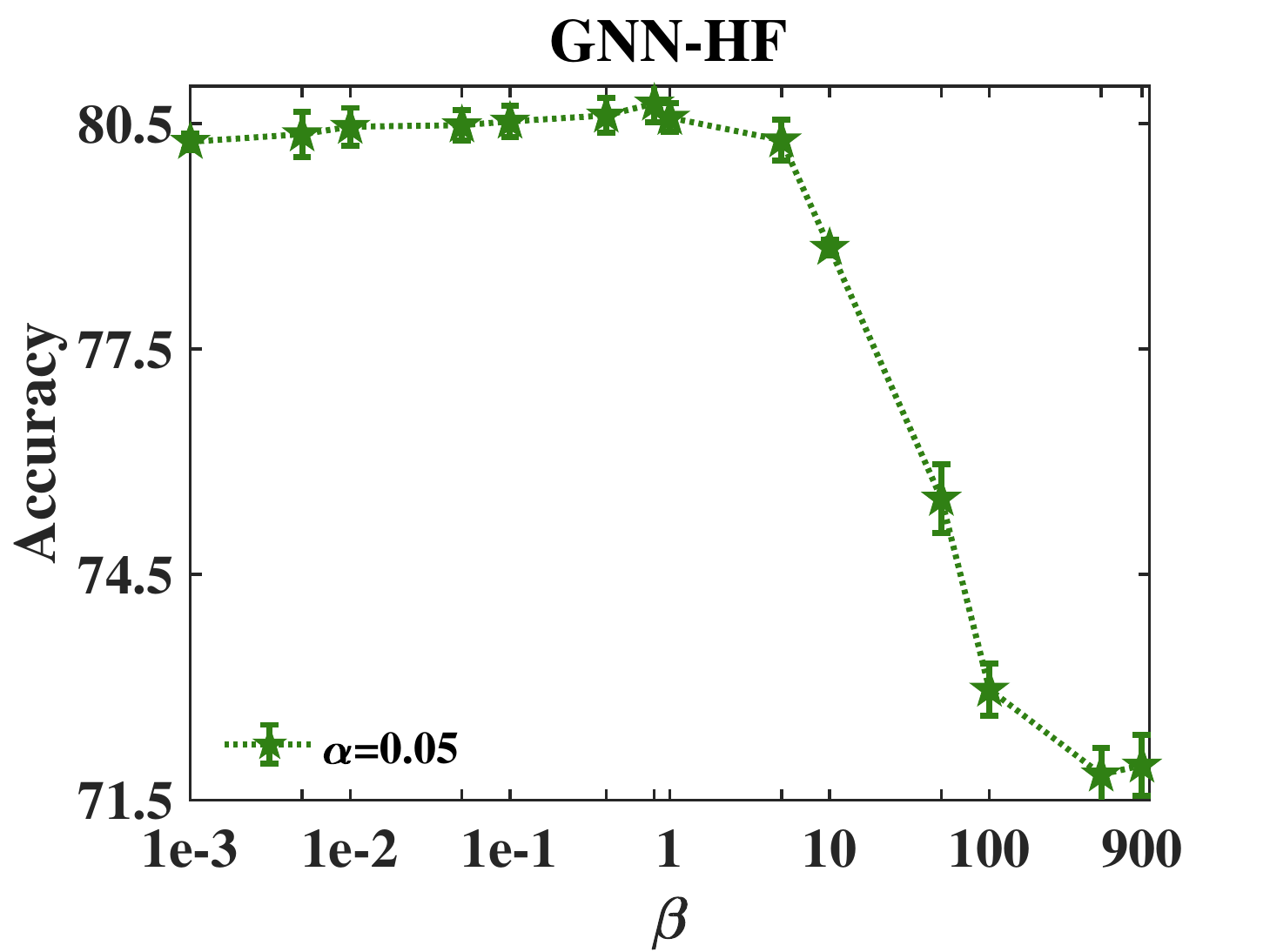}
	}
	\caption{Impact factor analysis with fixed $\alpha$ on Pubmed.}
	\label{parameter analysis}
\end{figure}

\subsection{Propagation Depth Analysis}
Because our proposed GNN-LF/HF flexibly consider extra filtering information during propagation and have high expressive power, here we further conduct experiments on GNN-LF/HF and other shallow/deep models with different propagation depths using three datasets. For all the models, we set the hidden size as 64 and tune hyperparameters reasonably to get the best performance. Note that because of the specific architecture of JKNet and IncepGCN in \cite{rong2020dropedge}, they are analyzed from 3 layers and IncepGCN on pubmed dataset faces the out of memory problem when the depth is bigger than 34. Figure \ref{layer} shows the accuracy with different propagation depths, and we have the following observations:

GNN-LF/HF and APPNP greatly alleviate the over-smoothing problem, since the performance does not drop when the depth increases. Furthermore, GNN-LF/HF are more flexible and more expressive with higher results under the influence of extra graph filters and three adjustable impact factors $\alpha$, $\mu$ and $\beta$. As analyzed before, the polynomial filter coefficients of GNN-LF/HF are further more expressive and flexible than APPNP, GCN or SGC, which is useful for mitigating over-smoothing problem. Accuracy breaks down seriously for GCN while it drops a little bit slowly for SGC, but both GCN and SGC face the over-smoothing problem since the fixed polynomial filter coefficients limit their expressive power. As for JKNet/IncepGCN, they are deep GNNs to alleviate over-smoothing problem but still have to face performance degradation when the propagation depth increases.

\subsection{Model Analysis} \label{subsec:: model analysis}

In this section, we analyze the performance of GNN-LF/HF with different impact factors: teleport probability $\alpha$, balance coefficient $\mu$ and balance coefficient $\beta$. In general, $\alpha$ adjusts the regularization term weight and has an effect on structural information during propagation; $\mu$ and $\beta$ focus on adjusting the balance weights between different filters and have effects on feature information during propagation.

We carefully adjust the value of three impact factors with GNN-LF-closed/GNN-HF-closed models on cora and citeseer datasets, and draw the contour map for accuracy in Figure \ref{parameter_cora/cite}. As can be seen: 1) For GNN-LF, $\alpha$ plays a more dominant influence than $\mu$. The classification accuracy exactly increases as $\alpha$ becomes larger, and with the continuous increase of $\alpha$, the results begin to drop. Generally, the best performance can be achieved when $\alpha \in [1e-2, 5e-1]$. On the other hand, the accuracy is relatively stable with different $\mu$, and generally speaking, $\mu$ is with a suitable range around $[0.6, 0.9]$. 2) For GNN-HF, $\alpha$ and $\beta$ both play dominant influence, the suitable weight range for $\alpha$ is also [1e-2, 5e-1] and lager $\beta$ may result in performance degradation. In general, our proposed GNN-LF/HF achieve stable and excellent performance within a wide changing range of these impact factors $\alpha$ and $\mu$ or $\beta$.

We then analyze the influence of balance coefficients $\mu$ and $\beta$ with the fixed $\alpha$ on pubmed dataset, shown in Figure \ref{parameter analysis}, and we have similar conclusions: For GNN-LF, the performance first increases stably as $\mu$ grows and then may show a slight drop after a certain threshold. The appropriate range for best performance is [0.6, 0.9]. For GNN-LF, the classification accuracy first increases and then drops with the rise of $\mu$ after a certain threshold.

\section{Conclusion}\label{sec::Conclusion}

The intrinsic relation for the propagation mechanisms of different GNNs is studied in this paper. We establish the connection between different GNNs and a flexible objective optimization framework. The proposed unified framework provides a global view on understanding and analyzing different GNNs, which further enables us to identify the weakness of current GNNs. Then we propose two novel GNNs with adjustable convolutional kernels showing low-pass and high-pass
filtering capabilities, and their excellent expressive power is analyzed as well. Extensive experiments well demonstrate the superior performance of these two GNNs over the state-of-the-art models on real world datasets.

\balance
\bibliographystyle{ACM-Reference-Format}
\bibliography{ref}

\appendix
\section{Proofs and Analysis}
\subsection{Proof of Theorem \ref{GC_operation}} \label{subsection::GC Operations}
With the objective Eq. (\ref{4_eq_ReGCN_object}), we have the following closed-form solution:
\begin{equation}\label{4_eq_ReGCN_proof1}
	\hat{\textbf Z} = (\textbf I + \tilde{\textbf{L}}) ^{-1} \textbf H.
\end{equation}
Similar to the analysis in \cite{nt2019revisiting}, we can decompose the matrix $(\textbf I + \tilde{\textbf{L}}) ^{-1}$ and get the first-order truncated form as:
\begin{equation}\label{4_eq_ReGCN_proof3}
	(\textbf I + \tilde{\textbf{L}})^{-1} \approx \textbf I - \tilde{\textbf{L}} = \hat{\tilde{\textbf{A}}}.
\end{equation}
In this way, we have the first-order approximation of the GC operation:
\begin{equation}\label{4_eq_ReGCN_proof4}
	\begin{aligned}
		\textbf{Z}^{(GC)} &= \hat{\tilde{\textbf{A}}} \textbf{H} = \hat{\tilde{\textbf{A}}} \textbf{X} \textbf{W}.
	\end{aligned}
\end{equation}
At this point, we provide another explanation on graph convolutional operation with the first-order approximation based on the framework.
\subsection{Proof of Theorem \ref{GNN_HF_convergence}} \label{sec::conver_HF}
\textbf{GNN-HF-iter} uses the iteration equation:
\begin{equation}\label{3_eq_iterform}
	\begin{aligned}
		\textbf Z^{(0)} = \frac{1}{\alpha \beta + 1} \textbf H &+ \frac{\beta}{\alpha \beta + 1} \tilde{\textbf{L}}\textbf H, \\
		\textbf Z^{(k+1)} = \frac{\alpha \beta - \alpha + 1}{\alpha \beta + 1} \hat{\tilde{\textbf{A}}} \textbf Z^{(k)} &+ \frac{\alpha}{\alpha \beta + 1} \textbf H + \frac{\alpha \beta}{\alpha \beta + 1} \tilde{\textbf{L}} \textbf H.
	\end{aligned}
\end{equation}
The corresponding closed form with the same $ \textbf H = f_\theta (\textbf{X}) $ is:
\begin{equation}\label{3_eq_closedform}
	\textbf{Z} =\big\{(\beta + 1/\alpha) \textbf I + (1 - \beta - 1/\alpha) \hat{\tilde{\textbf{A}}} \big\}^{-1} \{\textbf I + \beta \tilde{\textbf{L}} \}\textbf H.
\end{equation}
After the $K$-layer propagation using GNN-HF-iter, the corresponding expansion result can be written as:
\begin{equation}\label{3_eq_converproof0}
	\begin{aligned}
		\textbf Z^{(k)} = & \bigg\{(\frac{\alpha \beta - \alpha + 1}{\alpha \beta + 1})^k \hat{\tilde{\textbf{A}}}^k + \alpha \sum_{i=0}^{k-1}(\frac{\alpha \beta - \alpha + 1}{\alpha \beta + 1})^i \hat{\tilde{\textbf{A}}}^i\bigg\} \bigg\{\frac{1}{\alpha \beta + 1} \textbf H \\
		&+ \frac{\beta}{\alpha \beta + 1} \tilde{\textbf{L}} \textbf H\bigg\}, \\
	\end{aligned}
\end{equation}
where $\beta \in (0, \infty)$, $\alpha \in (0, 1]$ and $|\frac{\alpha \beta - \alpha + 1}{\alpha \beta + 1}|< 1 $. Similar with the proof process of \textbf{GNN-LF-iter} in Theorem \ref{GNN_LF_convergence}, we have the converging result as:
\begin{equation}\label{3_eq_converproof6}
	\begin{aligned}
		&\textbf Z^{(\infty)} = \big\{(\beta + 1/\alpha) \textbf I + (1 - \beta - 1/\alpha) \hat{\tilde{\textbf{A}}} \}^{-1} \{\textbf I + \beta \tilde{\textbf{L}} \big\}\textbf H,
	\end{aligned}
\end{equation}
which exactly is the Eq. (\ref{1_eq_V2design1}) for calculating GNN-HF-closed.
\subsection{Expressive Power Analysis} \label{sec::Polynomial Filter}
\textbf{Analysis of SGC.} \quad
As \cite{wu2019simplifying} points out, the $K$-layer graph convolutional operations or simplified graph convolutional operations act as the spectral polynomial filter of order $K$ with fixed coefficients. \cite{chen2020simple} proves that such fixed coefficients limit the expressive power of GCN and thus leads to over-smoothing. The $K$-order polynomial filter on graph signal \textbf{x} is:
\begin{equation}\label{2_eq_sgcfilter1}
	\textbf z^{(K)} = \hat{\tilde{\textbf{A}}}^{K} \textbf{x}  =  \big(\textbf I - \tilde{\textbf{L}}\big)^{K} \textbf{x}.
\end{equation}
By calculating the expansion of Eq. (\ref{2_eq_sgcfilter1}), we can conclude the $k$-th polynomial filtering term, denoted by $\theta_k \tilde{\textbf{L}}^k$.Then for the filter coefficients on $\tilde{\textbf{L}}^k, k \in [0, K]$, we have $\theta_{k} =  (-1)^k \binom{K}{k}$, which is a fixed constant for any k.

\textbf{Analysis of PPNP.} \quad As proved in \cite{klicpera2019predict}, PPNP or $K$-order APPNP $(K\rightarrow \infty)$ has the following expressive power:
\begin{equation}\label{1_eq_expressproof1}
	\textbf z ^ {(K)} = \big\{(1-\alpha)^K\hat{\tilde{\textbf{A}}}^K + \alpha \sum_{i=0}^{K-1}(1-\alpha)^{i} \hat{\tilde{\textbf{A}}}^{i}\big\} \textbf x.
\end{equation}
Since $(1-\alpha)^{\infty} \rightarrow 0$, we can rewrite it using the normalized graph Laplacian $\tilde{\textbf{L}}$ as:
\begin{equation}\label{1_eq_ppnpfilter1}
	\begin{aligned}
		\textbf z ^ {(K)} &= \alpha \sum_{i=0}^{K-1}\big(1-\alpha\big)^{i} \hat{\tilde{\textbf{A}}}^{i} \textbf x  = \alpha \sum_{i=0}^{K-1}\big(1-\alpha\big)^{i} \big(\textbf I - \tilde{\textbf{L}}\big)^{i} \textbf x.
	\end{aligned}
\end{equation}
Then we can calculate the expansion of Eq. (\ref{1_eq_ppnpfilter1}) and conclude the $k$-th polynomial filtering term, denoted by $\theta_k \tilde{\textbf{L}}^k$. Then for the filter coefficients on $\tilde{\textbf{L}}^k, k \in [0, K-1]$, we have:
\begin{equation}\label{1_eq_ppnpfilter3}
	\theta_{k} = \alpha \sum\limits_{i=k}^{K-1} (1-\alpha)^i (-1)^k\binom{i}{k}.
\end{equation}


\textbf{Analysis of GNN-HF.} \quad Similarly, taking the propagation result in Eq. (\ref{3_eq_converproof0}) with $                                                                                                                                                                                                                                                                                                                                                                                                                                                                                                                                                                                                                             K \rightarrow \infty$, we can also have the corresponding filtering expression:
\begin{equation}\label{1_eq_V2filter1}
	\begin{aligned}
		\textbf z^{(K)} &= \Big\{\alpha \sum_{i=0}^{K-1}\big(\frac{\alpha \beta - \alpha + 1}{\alpha \beta + 1}\big)^i \hat{\tilde{\textbf{A}}}^i\Big\} \Big\{\frac{1}{\alpha \beta + 1} \textbf x + \frac{\beta}{\alpha \beta + 1} \tilde{\textbf{L}} \textbf x\Big\} \\
		&= \frac{\alpha(\beta + 1)}{\alpha \beta + 1} \Big\{\sum_{i=0}^{K-1}\big(\frac{\alpha \beta - \alpha + 1}{\alpha \beta + 1}\big)^i \hat{\tilde{\textbf{A}}}^i \Big\} \textbf x - \frac{\alpha \beta}{\alpha \beta - \alpha + 1} \bm{\cdot} \\
		&\Big\{\sum_{i=1}^{K}\big(\frac{\alpha \beta - \alpha + 1}{\alpha \beta + 1}\big)^i \hat{\tilde{\textbf{A}}}^i \Big\}\textbf x \\
		&= \frac{\alpha(\beta + 1)}{\alpha \beta + 1} \Big\{\sum_{i=0}^{K-1}\big(\frac{\alpha \beta - \alpha + 1}{\alpha \beta + 1}\big)^i \big({\textbf I - \tilde{\textbf{L}}}\big)^i \Big\} \textbf x - \frac{\alpha \beta}{\alpha \beta - \alpha + 1} \bm{\cdot} \\
		&\Big\{\sum_{i=1}^{K}\big(\frac{\alpha \beta - \alpha + 1}{\alpha \beta + 1}\big)^i \big({\textbf I - \tilde{\textbf{L}}}\big)^i \Big\}\textbf x.
	\end{aligned}
\end{equation}

Then for the filter coefficients on $\tilde{\textbf{L}}^k, k \in [0, K]$, we have the following conclusions:
\begin{itemize}
	\item[1)] \textit{Filter coefficients for} $\tilde{\textbf{L}}^0$:
	\begin{equation}\label{1_eq_V2ilter2}
		\begin{aligned}
			&\theta_{0} = \frac{\alpha(\beta + 1)(\alpha \beta - \alpha + 1)}{(\alpha \beta + 1)^2} - \frac{\alpha \beta (\alpha \beta -\alpha + 1)^{K-1}}{(\alpha \beta + 1)^{K}} + \sum\limits_{j=1}^{K-1} \delta_j\binom{j}{0}, \\
			&\delta_j = \big\{\frac{\alpha(\beta + 1)}{\alpha \beta + 1}-\frac{\alpha \beta}{\alpha \beta - \alpha + 1} \big\} (\frac{\alpha \beta - \alpha + 1}{\alpha \beta + 1}) ^j.
		\end{aligned}
	\end{equation}
	\item[2)] \textit{Filter coefficients for} $\tilde{\textbf{L}}^k, k \in [1, K-1]$:
	\begin{equation}\label{1_eq_V2filter3}
		\begin{aligned}
			\theta_{0} &= \sum\limits_{j=i}^{K} \delta_j(-1)^i\binom{j}{i}, \\
			\delta_j &= \big\{\frac{\alpha(\beta + 1)}{\alpha \beta + 1}-\frac{\alpha \beta}{\alpha \beta - \alpha + 1} \big\} (\frac{\alpha \beta - \alpha + 1}{\alpha \beta + 1}) ^j.
		\end{aligned}
	\end{equation}
	\item[3)] \textit{Filter coefficients for} $\tilde{\textbf{L}}^K$:
	\begin{equation}\label{1_eq_V2filter4}
		\theta_{K} = - \frac{\alpha \beta(\alpha \beta - \alpha + 1)^{K-1}}{(\alpha \beta + 1)^{K}}(-1)^K \binom{K}{K}.
	\end{equation}
\end{itemize}

\balance
\end{document}